\documentclass[11pt]{article}

\usepackage[english]{babel}
\usepackage[T1]{fontenc}
\usepackage[margin=1in]{geometry} 
\usepackage{tikz}
\usepackage{amsmath}
\usepackage{amsthm}
\usepackage{amssymb}
\usepackage{mathtools}
\usepackage{graphicx}
\usepackage{subcaption}
\usepackage{amssymb}
\usepackage{bm}
\usepackage{xspace}
\usepackage{xcolor}
\usepackage{comment}
\usepackage{float} 
\usepackage{thmtools}
\usepackage{thm-restate}

\newtheorem{theorem}{Theorem}[section]
\newtheorem{lemma}[theorem]{Lemma}
\newtheorem{definition}[theorem]{Definition}

\newtheorem{algorithm}[theorem]{Algorithm}
\newtheorem{remark}[theorem]{Remark}

\newtheorem{problem}[theorem]{Problem}

\DeclareMathOperator{\A}{\mathcal{A}\xspace}

\DeclareMathOperator{\M}{\mathcal{M}}

\DeclareMathOperator{\G}{\mathcal{G}}
\DeclareMathOperator{\I}{\mathcal{I}}

\DeclareMathOperator{\x}{\mathbf{x}\xspace}
\DeclareMathOperator{\z}{\mathbf{z}\xspace}

\DeclareMathOperator{\w}{\mathbf{w}\xspace}

\usetikzlibrary{arrows.meta}

\makeatletter

\pgfdeclarearrow{
  name = ipe _linear,
  defaults = {
    length = +1bp,
    width  = +.666bp,
    line width = +0pt 1,
  },
  setup code = {
    \pgfarrowssetbackend{0pt}
    \pgfarrowssetvisualbackend{
      \pgfarrowlength\advance\pgf@x by-.5\pgfarrowlinewidth}
    \pgfarrowssetlineend{\pgfarrowlength}
    \ifpgfarrowreversed
      \pgfarrowssetlineend{\pgfarrowlength\advance\pgf@x by-.5\pgfarrowlinewidth}
    \fi
    \pgfarrowssettipend{\pgfarrowlength}
    \pgfarrowshullpoint{\pgfarrowlength}{0pt}
    \pgfarrowsupperhullpoint{0pt}{.5\pgfarrowwidth}
    \pgfarrowssavethe\pgfarrowlinewidth
    \pgfarrowssavethe\pgfarrowlength
    \pgfarrowssavethe\pgfarrowwidth
  },
  drawing code = {
    \pgfsetdash{}{+0pt}
    \ifdim\pgfarrowlinewidth=\pgflinewidth\else\pgfsetlinewidth{+\pgfarrowlinewidth}\fi
    \pgfpathmoveto{\pgfqpoint{0pt}{.5\pgfarrowwidth}}
    \pgfpathlineto{\pgfqpoint{\pgfarrowlength}{0pt}}
    \pgfpathlineto{\pgfqpoint{0pt}{-.5\pgfarrowwidth}}
    \pgfusepathqstroke
  },
  parameters = {
    \the\pgfarrowlinewidth,%
    \the\pgfarrowlength,%
    \the\pgfarrowwidth,%
  },
}

\pgfdeclarearrow{
  name = ipe _pointed,
  defaults = {
    length = +1bp,
    width  = +.666bp,
    inset  = +.2bp,
    line width = +0pt 1,
  },
  setup code = {
    \pgfarrowssetbackend{0pt}
    \pgfarrowssetvisualbackend{\pgfarrowinset}
    \pgfarrowssetlineend{\pgfarrowinset}
    \ifpgfarrowreversed
      \pgfarrowssetlineend{\pgfarrowlength}
    \fi
    \pgfarrowssettipend{\pgfarrowlength}
    \pgfarrowshullpoint{\pgfarrowlength}{0pt}
    \pgfarrowsupperhullpoint{0pt}{.5\pgfarrowwidth}
    \pgfarrowshullpoint{\pgfarrowinset}{0pt}
    \pgfarrowssavethe\pgfarrowinset
    \pgfarrowssavethe\pgfarrowlinewidth
    \pgfarrowssavethe\pgfarrowlength
    \pgfarrowssavethe\pgfarrowwidth
  },
  drawing code = {
    \pgfsetdash{}{+0pt}
    \ifdim\pgfarrowlinewidth=\pgflinewidth\else\pgfsetlinewidth{+\pgfarrowlinewidth}\fi
    \pgfpathmoveto{\pgfqpoint{\pgfarrowlength}{0pt}}
    \pgfpathlineto{\pgfqpoint{0pt}{.5\pgfarrowwidth}}
    \pgfpathlineto{\pgfqpoint{\pgfarrowinset}{0pt}}
    \pgfpathlineto{\pgfqpoint{0pt}{-.5\pgfarrowwidth}}
    \pgfpathclose
    \ifpgfarrowopen
      \pgfusepathqstroke
    \else
      \ifdim\pgfarrowlinewidth>0pt\pgfusepathqfillstroke\else\pgfusepathqfill\fi
    \fi
  },
  parameters = {
    \the\pgfarrowlinewidth,%
    \the\pgfarrowlength,%
    \the\pgfarrowwidth,%
    \the\pgfarrowinset,%
    \ifpgfarrowopen o\fi%
  },
}

\newdimen\ipeminipagewidth

\tikzstyle{ipe import} = [
  x=1bp, y=1bp,
  ipe node stretch/.store in=\ipenodestretch,
  ipe stretch normal/.style={ipe node stretch=1},
  ipe stretch normal,
  ipe node/.style={
    anchor=base west, inner sep=0, outer sep=0, scale=\ipenodestretch
  },
  ipe mark scale/.store in=\ipemarkscale,
  ipe mark tiny/.style={ipe mark scale=1.1},
  ipe mark small/.style={ipe mark scale=2},
  ipe mark normal/.style={ipe mark scale=3},
  ipe mark large/.style={ipe mark scale=5},
  ipe mark normal, % Set default
  ipe circle/.pic={
    \draw[line width=0.2*\ipemarkscale]
      (0,0) circle[radius=0.5*\ipemarkscale];
    \coordinate () at (0,0);
  },
  ipe disk/.pic={
    \fill (0,0) circle[radius=0.6*\ipemarkscale];
    \coordinate () at (0,0);
  },
  ipe fdisk/.pic={
    \filldraw[line width=0.2*\ipemarkscale]
      (0,0) circle[radius=0.5*\ipemarkscale];
    \coordinate () at (0,0);
  },
  ipe box/.pic={
    \draw[line width=0.2*\ipemarkscale, line join=miter]
      (-.5*\ipemarkscale,-.5*\ipemarkscale) rectangle
      ( .5*\ipemarkscale, .5*\ipemarkscale);
    \coordinate () at (0,0);
  },
  ipe square/.pic={
    \fill
      (-.6*\ipemarkscale,-.6*\ipemarkscale) rectangle
      ( .6*\ipemarkscale, .6*\ipemarkscale);
    \coordinate () at (0,0);
  },
  ipe fsquare/.pic={
    \filldraw[line width=0.2*\ipemarkscale, line join=miter]
      (-.5*\ipemarkscale,-.5*\ipemarkscale) rectangle
      ( .5*\ipemarkscale, .5*\ipemarkscale);
    \coordinate () at (0,0);
  },
  ipe cross/.pic={
    \draw[line width=0.2*\ipemarkscale, line cap=butt]
      (-.5*\ipemarkscale,-.5*\ipemarkscale) --
      ( .5*\ipemarkscale, .5*\ipemarkscale)
      (-.5*\ipemarkscale, .5*\ipemarkscale) --
      ( .5*\ipemarkscale,-.5*\ipemarkscale);
    \coordinate () at (0,0);
  },
  /pgf/arrow keys/.cd,
  ipe arrow normal/.style={scale=1},
  ipe arrow tiny/.style={scale=.4},
  ipe arrow small/.style={scale=.7},
  ipe arrow large/.style={scale=1.4},
  ipe arrow normal,
  /tikz/.cd,
  ipe arrows/.style={
    ipe normal/.tip={
      ipe _pointed[length=1bp, width=.666bp, inset=0bp,
                   quick, ipe arrow normal]},
    ipe pointed/.tip={
      ipe _pointed[length=1bp, width=.666bp, inset=0.2bp,
                   quick, ipe arrow normal]},
    ipe linear/.tip={
      ipe _linear[length = 1bp, width=.666bp,
                  ipe arrow normal, quick]},
    ipe fnormal/.tip={ipe normal[fill=white]},
    ipe fpointed/.tip={ipe pointed[fill=white]},
    ipe double/.tip={ipe normal[] ipe normal},
    ipe fdouble/.tip={ipe fnormal[] ipe fnormal},
    ipe arc/.tip={ipe normal},
    ipe farc/.tip={ipe fnormal},
    ipe ptarc/.tip={ipe pointed},
    ipe fptarc/.tip={ipe fpointed},
  },
  ipe arrows, % Set default sizes
]

\tikzset{
  rgb color/.code args={#1=#2}{%
    \definecolor{tempcolor-#1}{rgb}{#2}%
    \tikzset{#1=tempcolor-#1}%
  },
}

\makeatother

\DeclareMathOperator{\rk}{rk\xspace}
\DeclareMathOperator{\off}{\mathbf{r}\xspace}

\newcommand\event[1]{\mathop{\mathcal{X}\left(#1\right)}}

\newcommand\E[1]{\mathop{\mathbb{E}\left[#1\right]}}
\newcommand\Ex[2]{\mathop{\underset{#1}{\mathbb{E}}\left[#2\right]}}

\newcommand\Pro[1]{\mathop{\mathbb{P}\left(#1\right)}}
\newcommand\Prob[2]{\mathop{\underset{#1}{\mathbb{P}}\left(#2\right)}}

\DeclareMathOperator{\Rew}{\mathcal{R}\xspace}
\DeclareMathOperator*{\maximize}{\bf{maximize:}}

\newcommand{\rsm}{{RSW}\xspace
}

\newcommand{\mbb}{{MBB}\xspace
}

\usepackage{hyperref}
\hypersetup{
    colorlinks=true,
    linkcolor=black,
    citecolor=blue
}
\usepackage{float}
\usepackage{bookmark}
\makeatletter
\newcommand{\printfnsymbol}[1]{%
  \textsuperscript{\@fnsymbol{#1}}%
}
\makeatother

\newcommand{\ig}{\textsc{interleaved-greedy}\xspace}
\newcommand{\IG}{IG\xspace}

\newcommand{\UCB}{IB\xspace}

\newcommand{\is}{\textsc{interleaved-submodular}\xspace}
\newcommand{\IS}{IS\xspace}
\newcommand{\OPT}{\textsc{OPT}\xspace}

\newcommand{\ucb}{\textsc{interleaved-ucb}\xspace}

\title{Recurrent Submodular Welfare and \\
Matroid Blocking Bandits}
\author{
Orestis Papadigenopoulos\\
\texttt{papadig@cs.utexas.edu}\\
Department of Computer Science\\
The University of Texas at Austin, USA.
\and
Constantine Caramanis\\
\texttt{constantine@utexas.edu}\\
Department of Electrical and Computer Engineering\\
The University of Texas at Austin, USA.
}
\date{\today}

\begin{document}

\maketitle

\begin{abstract}
A recent line of research focuses on the study of the stochastic multi-armed bandits problem (MAB), in the case where temporal correlations of specific structure are imposed between the player's actions and the reward distributions of the arms (Kleinberg and Immorlica [FOCS18], Basu et al. [NeurIPS19]). As opposed to the standard MAB setting, where the optimal solution in hindsight can be trivially characterized, these correlations lead to (sub-)optimal solutions that exhibit interesting dynamical patterns -- a phenomenon that yields new challenges both from an algorithmic as well as a learning perspective. In this work, we extend the above direction to a combinatorial bandit setting and study a variant of stochastic MAB, where arms are subject to matroid constraints and each arm becomes unavailable (blocked) for a fixed number of rounds after each play. A natural common generalization of the state-of-the-art for blocking bandits, and that for matroid bandits, yields a $(1-\frac{1}{e})$-approximation for partition matroids, yet it only guarantees a $\frac{1}{2}$-approximation for general matroids. In this paper we develop new algorithmic ideas that allow us to obtain a polynomial-time $(1 - \frac{1}{e})$-approximation algorithm (asymptotically and in expectation) for any matroid, and thus to control the $(1-\frac{1}{e})$-approximate regret. A key ingredient is the technique of correlated (interleaved) scheduling. Along the way, we discover an interesting connection to a variant of Submodular Welfare Maximization, for which we provide (asymptotically) matching upper and lower approximability bounds.

\end{abstract}

\pagenumbering{arabic}
%\linenumbers

\section{Introduction}
The {\em multi-armed bandits} (MAB) model has proven itself to be a successful mathematical framework for studying sequential decision making problems in environments that are initially unexplored by the decision maker. Since its first introduction (see \cite{T33} and later \cite{LR85}), the framework has been extensively studied as a simple yet powerful abstraction of the trade-off between {\em exploration} and {\em exploitation}, ubiquitous in a plethora of applications ranging from online advertising/recommendation systems to resource allocation and dynamic pricing (see \cite{BB12,LS18} and references therein). In the {\em stochastic} MAB setting \cite{LR85} the decision maker sequentially chooses among a set of available {\em actions} (or {\em arms}), each associated with an i.i.d. stochastic reward, while in the {\em combinatorial} MAB setting \cite{CBL12}, a subset of the arms can be selected at each round, subject to feasibility constraints.

Despite the large number of variants of the MAB model that have been introduced, the majority of the results comply with the common assumption that playing an action does not alter the environment, namely, the reward distributions of the subsequent rounds (with notable exceptions discussed below). Only recently, researchers have focused their attention on settings where temporal dependencies of specific structure are imposed between the player's actions and the reward distributions~\cite{KI18, CCB19,BSSS19, PBG19,BPCS20}. In \cite{KI18}, Kleinberg and Immorlica consider the setting of {\em recharging bandits}, where the expected reward of each arm is a concave and weakly increasing function of the time passed since its last play, modeling in that way scenarios of local performance loss. 
%In the same spirit, Cella and Cesa-Bianchi \cite{CCB19} consider a related setting, assuming a parametric form for the reward functions which are non-decreasing and bounded by $1$. 
In a similar spirit, Basu et al.\ \cite{BSSS19} consider the problem of {\em blocking bandits}, in which case once an arm is played at some round, it cannot be played again (i.e., it becomes blocked) for a fixed number of consecutive rounds. Notice that all the aforementioned examples are variations of the stochastic MAB setting, where the decision maker plays (at most) one arm per time step.

When combinatorial constraints and time dynamics come together, the result is a much richer and more challenging setting, precisely because their interplay creates a complex dynamical structure. Indeed, in the standard combinatorial bandits setting, the optimal solution in hindsight is to consistently play the feasible subset of arms of maximum expected reward. However, in the presence of local temporal constraints on the arms, an optimal (or even suboptimal) solution cannot be trivially characterized-- a fact that significantly complicates the analysis, both from the algorithmic as well as from the learning perspective. In this work, we study %action-reward temporal dependencies in 
the following bandit setting-- a common generalization of {\em matroid bandits}, introduced by Kveton et al.\
\cite{KWAEE14}, and blocking bandits~\cite{BSSS19}:

\begin{problem}
[{Matroid Blocking Bandits (MBB)}] 
We consider a set $\A$ of $k$ {\em arms}, a {\em matroid} $\M=(\A,\I)$, and an unknown time horizon of $T$ rounds. Each arm $i \in \A$ is associated with an unknown bounded {\em reward} distribution of mean $\mu_i$, and with a known deterministic {\em delay} $d_i$, such that whenever an action $i$ is played at some round, it cannot be played again for the next $d_i-1$ rounds. At each time step, the {\em player} pulls a subset of the available (i.e., not blocked) arms restricted to be an independent set of $\M$. Subsequently, she observes the reward realization of each arm played ({\em semi-bandit} feedback) and collects their sum. The goal of the player is to maximize her expected cumulative reward over $T$ rounds.
\end{problem}
 
The above model captures a number of applications, varying from team formation to ad placement, when arms represent actions that cannot be played repeatedly without restriction. As a concrete example, consider a recommendation system that repeatedly suggests a variety of products (e.g., songs, movies, books) to a user. The need for diversity on the collection of suggested products (arms), to capture different aspects of user's preferences, can be modeled as a linear matroid. Further, the blocking constraints preclude the incessant recommendation of the same product (which can be detrimental, as the product might be perceived as a ``spam''), while the maximum rate of recommendation (controlled by the delay) might depend on factors such as popularity, promotion and more. Finally, the expected reward of each product is the probability of purchasing (or clicking).

From a technical viewpoint, the \mbb problem is already NP-hard for the simple case of a uniform rank-1 matroid (see Theorem 2.1 in \cite{SST09}), even in the {\em full-information} setting, where the reward distributions are known to the player a priori. The natural common generalization of the algorithms in \cite{BSSS19,KWAEE14}, computes and plays, at each time step, an independent set of maximum mean reward consisting of the available elements. While the above strategy is a $\left(1-\frac{1}{e}\right)$-approximation asymptotically (that is, for $T \to \infty$) for partition matroids, unfortunately, it only guarantees a $\frac{1}{2}$-approximation for general matroids and this guarantee is tight (see Appendix \ref{appendix:tightexample} for an example). A natural question that arises is whether a $\left(1-\frac{1}{e}\right)$-approximation is possible for any matroid. 

The main result of this paper shows that this is indeed possible. Along the way, we identify that the key insight (and also the weak point of the naive $\frac{1}{2}$-approximation) is the underlying {\em diminishing returns} property hidden in the matroid structure. In particular, we discover an interesting connection of \mbb to the following problem of interest in its own right:

\begin{problem}[Recurrent Submodular Welfare (RSW)]
We consider a monotone (non-decreasing) submodular function $f:2^{\A} \rightarrow \mathbb{R}_{\geq 0}$ over a universe $\A$ and a time horizon $T$. At each round $t \in [T]$ we choose a subset $\A_t \subseteq \A$ and collect a reward $f(\A_t)$. However, using an element $i \in \A$ at some round $t \in [T]$ makes it unavailable (i.e., blocked) for a fixed and known number of $d_i-1$ subsequent rounds, namely, during the interval $[t, t + d_i -1]$. The objective is to maximize $\sum_{t \in [T]} f(\A_t)$, subject to the blocking constraints, within a (potentially unknown) time horizon $T$.
\end{problem}

For the above model, which can be thought of as a variant of {\em Submodular Welfare Maximization}~\cite{Von08}, we provide an efficient randomized $\left(1 - \frac{1}{e}\right)$-approximation (asymptotically), accompanied by a matching hardness result. Note that the \rsm problem is a very natural model, capturing applications of submodular maximization in repeating scenarios, where the elements cannot be constantly used without restriction. As an example, consider the process of renting goods to a stream of customers with identical submodular utility functions modeling their satisfaction.

As we show, our approach for the \rsm problem immediately implies an algorithm of the same approximation guarantee for the full-information case of \mbb and, additionally, it has important implications for the {\em bandit} setting, where the reward distributions are initially unknown. The standard goal in this case is to provide a (sublinear in the time horizon) upper bound on the {\em regret}, namely, the difference between the expected reward of a bandit algorithm and a (near-)optimal algorithm, due to the initial lack of knowledge of the former\footnote{
In fact, we upper bound the $\left(1 - \frac{1}{e}\right)${-(approximate) regret}, defined as the difference between $\left(1 - \frac{1}{e}\right) \OPT(T)$ and the expected reward collected by a bandit algorithm. The notion of $\alpha$-regret is widely used in the combinatorial bandits literature~\cite{CWYW16,WC17} for combinatorial problems where an efficient algorithm does not exist, and, thus, any efficient algorithm would inevitably suffer linear regret in standard definition (where $\alpha = 1$).
}.

\subsection{Related Work}
The \mbb model belongs to the family of stochastic {\em non-stationary} bandits, given that the reward distributions of the arms can change over time. Significant members of this family are {\em restless bandits} \cite{Whittle88, GMS10}, where the reward distribution of each arm changes at each time step, and {\em rested bandits} \cite{Gittins79, TL12}, where the distribution changes only when the arm is played. For the setting of restless bandits and without further assumptions on the transition functions, it is PSPACE-hard to even approximate the optimal solution \cite{PT99}. Our model differs from the above cases as we consider a transition function of special form and the transitions can occur both during playing and not playing an arm. In addition, the \mbb model falls into the category of Markov Decision Processes (MDPs) with deterministic transitions and stochastic rewards, but requires an exponential (in the size of the arms) state space, which makes this approach inefficient in practice.

A recent line of research focuses on non-stationary models in the case where each reward distribution is a special function of the player's actions \cite{ CCB19, PBG19,BPCS20}. In \cite{BSSS19}, Basu et al. provide a greedy $\left(1-\frac{1}{e}\right)$-approximation for the full-information case of the blocking bandits problem (a special case of the \mbb model for a uniform rank-1 matroid). As we have already mentioned, generalizing their strategy to the \mbb problem fails to provide the same guarantee for general matroids. In the bandit setting, where the reward distributions are initially unknown, the authors have to overcome the burden of characterizing a (sub)optimal solution, where the rate of mean collected reward exhibits significant fluctuations over time. The key insight is to observe that every time the full-information algorithm plays an arm, its bandit variant, which relies on estimations of the mean rewards, has at least one chance of playing the same arm. However, this key coupling argument, that enables sublinear regret bounds, becomes significantly more involved in the presence of matroid constraints. 

In \cite{KI18}, Kleinberg and Immorlica study the case of recharging bandits. Their approach first computes the ``optimal'' playing frequency ${1/x_i}$ of each arm $i$ via a mathematical formulation. In order to play each arm with this frequency, they develop the technique of {\em interleaved rounding}, where they associate each arm $i$ with a sequence of real numbers $\{\frac{\alpha_i +k}{x_i}\}_{k \in \mathbb{N}}$, with $\alpha_i \sim U[0,1]$. Then, the arms are played sequentially in the same order they appear on the real line. This novel rounding technique exhibits reduced variance and, thus, an improved approximation guarantee comparing to other natural approaches such as independent randomized rounding.

A rich body of research on combinatorial bandits \cite{CTPL15, CWYW16,WWFLLL16,KWAS15,KWAS15b, WC17} focuses on bandit optimization problems over general combinatorial structures.
In \cite{KWAEE14}, Kveton et al.\ consider the problem of stochastic combinatorial bandits where the underlying feasible set is a matroid defined over the ground set of arms. At each round, the player pulls an independent subset of arms and collects their realized rewards, assuming {\em semi-bandit} feedback (as opposed to the pure exploration {\em full-feedback} variant studied in \cite{CGL16}). The authors develop a greedy algorithm based on the Upper Confidence Bound (UCB) method \cite{ACBF02}, while they exploit well-known exchange properties of matroids for achieving optimal regret bounds. Their approach relies on the fact that the optimal solution in hindsight is fixed throughout the time horizon-- a fact that is no longer true in the presence of blocking constraints. Additional lines of research that are related to, yet incompatible with, our problem are {\em bandits with knapsacks} \cite{BKS18,SA18} or {\em with budgets} \cite{CJS15, Sliv13}, and {\em sleeping bandits}~\cite{KNMS10}.

The \mbb model is also related to the literature on {\em periodic scheduling} \cite{BNLT07, BNBNS02}. In \cite{SST09}, Sgall et al.\ consider the problem of periodically scheduling jobs on a set of machines. Each job is associated with a {\em processing time}, during which it occupies the machine it is executed on, a {\em vacation time}, namely, a minimum time required after its completion in order to be rescheduled, and a {\em reward}. It is not hard to see that the case of unit processing times is a special case of \mbb with a uniform matroid of rank equal to the number of machines, under the objective of maximizing the total reward.
Further, it is known~\cite{BSSS19} that the rank-1 case of \mbb generalizes the {\em Pinwheel Scheduling} problem~\cite{HMRTV89}: Given $k$ colors associated a set of integers $\{d_i\}_{i \in [k]}$, such that $\sum_{i \in [k]} \frac{1}{d_i} = 1$, decide whether there is a coloring of the natural numbers $\nu:\mathbb{N} \to [k]$ such that every color $i \in [k]$ appears at least once every $d_i$ numbers. As it is proved in~\cite{JL14}, the above problem does not admit a pseudopolynomial time algorithm unless SAT can be solved by a randomized algorithm in expected quasi-polynomial
time.

Finally, the \rsm problem is closely related to the problem of {\em Submodular Welfare Maximization} (SWM) \cite{Von08,MSV08,KLMM05,FV10}: Given $k$ items and $m$ players, each associated with a monotone submodular utility function $u_i:2^{[k]}\rightarrow \mathbb{R}_{\geq 0}$, the goal is to partition the elements into $m$ sets $S_1, \dots, S_m$, one for each player, such that to maximize $\sum_{i \in [m]} u_i(S_i)$. Specifically, \rsm can be thought of as a version of the SWM problem, when the items are distributed to a (possibly infinite) stream of players with identical utilities, and each item can be reused after some fixed time period (note that this is different than the online setting in~\cite{KMZ15}). Interestingly, as noted in \cite{Von08}, the SWM problem with identical utilities is {\em approximation resistant} in the sense that allocating the items to the players uniformly at random achieves the optimal approximation guarantee of $\left(1 - \frac{1}{e}\right)$ for this setting.

\subsection{Our Contributions}

We first focus on the full-information variant of \mbb, where the mean rewards of the arms are known to the player a priori. We assume that the player has access to the matroid $\M$ via an independence oracle and knowledge of the arms' fixed delays, yet she is oblivious to the time horizon $T$. In this sense, she plays {\em online}. An interesting aspect of dynamics, as illustrated in \cite{KI18, BSSS19, BPCS20}, is that one needs to guarantee, via {\em scheduling}, that each arm is roughly played at a frequency close to its ``optimal'' rate. This is particularly important in the presence of ``hard'' blocking constraints, where no reward can be obtained by a blocked arm. 

In order to address the above scheduling problem, we propose a particular  ``decoupled'' {\em two-phase strategy}. We refer to each phase as Player A and Player B. Initially, Player A decides on a schedule that determines arm availability, namely, a subset of rounds where each arm is allowed to be played. Subsequently, Player B chooses a subset of available arms that maximizes the total expected reward, subject to the matroid constraints. In order to completely decouple the two phases, the availability schedule produced by Player A is never affected by which arms are eventually chosen by Player B (that is, it is impossible for Player B to violate the blocking constraints). 

In the case where Player B knows the expected rewards of the arms and due to the above decoupling property, his optimal strategy (given any availability schedule) can be easily characterized: Since the arms of each round are subject to matroid constraints, Player B achieves his goal by playing a maximum expected reward independent set among the available arms of each round, which can be computed efficiently using the greedy algorithm for matroids. Thus, the role of Player A becomes to choose an availability schedule that maximizes the total reward, knowing that Player B will behave exactly as described above. The key observation is that the solution computed by Player B at each round, corresponds to the {\em weighted rank function} of the matroid evaluated on the set of available arms of the round. More importantly, it can be proved that this function is monotone submodular and, hence, Player A's task is a special case of the \rsm problem. 

Focusing our attention on the \rsm problem, any ``good'' solution should guarantee that each element $i \in \A$ is selected a fraction of the time close to $\frac{1}{d_i}$ (the maximum possible), where $d_i$ is the delay. However, a naive randomized approach that selects (if available) each element $i$ with probability $\frac{1}{d_i}$ independently at each round, can be as bad as a $(1 - e^{-\frac{1}{2}}) \approx 0.393$-approximation. Instead, motivated by the rounding technique of Kleinberg and Immorlica \cite{KI18}, we develop a (time-)correlated sampling strategy, which we call {\em interleaved scheduling}. While our technique is based on the same principle of transforming (randomly interleaved) sequences of real numbers into a feasible schedule, our implementation is, to the best of our knowledge, novel. Indeed, as opposed to \cite{KI18}, we additionally face the issue of scheduling more than one arms per round and the fact that our ``hard'' blocking constraints are particularly sensitive to the variance of the produced schedule. Using our technique, we construct a polynomial-time randomized algorithm, named \is (\IS), that achieves the following guarantee for \rsm:

\begin{restatable}{theorem}{restateinterleavedSubmodular}\label{thm:interleavedSubmodular}
The expected reward collected by \is over $T$ rounds, $\Rew^{IS}(T)$, is at least
$
\left(1 - \frac{1}{e} \right) \OPT(T) - \mathcal{O}(d_{\max} f(\A))
$, where $\OPT(T)$ is the optimal reward of \rsm for $T$ rounds and $d_{\max} = \max_{i \in \A} d_i$ is the maximum delay of the instance.
\end{restatable}

The proof of the above guarantee relies on the construction of a {\em convex program} (CP), based on the {\em concave closure} of $f$ (see below), that yields an (approximate up to an additive term) upper bound on the optimal reward. Although our algorithm never computes an optimal solution to this convex program, it allows us to compare its expected collected reward with the optimal solution of CP, leveraging known results on the {\em correlation gap} of submodular functions. As we show via a reduction from the SWM problem with identical utilities, the $\left(1 - \frac{1}{e}\right)$ term in the above guarantee is asymptotically the best possible, unless $\text{P} = \text{NP}$; further, the additive term results from the fact that our algorithm is oblivious to the time horizon $T$.

We now turn our attention to the {\em bandit setting} of \mbb, where the mean rewards are initially unknown. Our interleaved scheduling method exhibits an additional property: {\em It does not rely on the monotone submodular function itself}, a fact that is particularly important for the bandit setting. Indeed, in the full-information setting Player B computes a maximum expected reward independent set at each round, for any availability schedule provided by Player A. In the bandit setting, however, the reward distributions are not a priori known and, thus, must be learned. Nevertheless, we do not need to wait to learn these distributions to find a good availability schedule. This allows us to make a natural coupling between the strategy of Player B in the bandit and in the full-information case and, thus, to compare the expected reward collected ``pointwise'', assuming a fixed common availability schedule. We remark that the above coupling is very different than the one in \cite{BSSS19}, as ours is independent of the trajectory of the observed rewards.

The above analysis allows us to develop a bandit algorithm for \mbb based on the UCB method, called \ucb (\UCB). Specifically, given any availability schedule provided by Player A (independently of the rewards) and in increasing order of rounds, Player B greedily computes a maximal independent set consisting the available arms of each round, based on estimates (known as UCB indices) of the mean rewards. In order to analyze the regret, we use the independence of the availability schedule in combination with the {\em strong basis exchange} property of matroids. This allows us to decompose the overall regret of our algorithm into contributions from each individual arm. Once we have established this regret decomposition, we can bound the individual regret attributed to each arm using more standard UCB type arguments \cite{KWAEE14}, leading to the following guarantee: 

\begin{restatable}{theorem}{restateApproxRegret}\label{thm:approxregret}
The expected reward collected by \ucb in $T$ rounds, $\Rew^{\UCB}(T)$, for $k$ arms, a matroid of rank $r = \rk(\M)$ and maximum delay $d_{\max}$ is at least
\begin{align*}
    \left(1-\frac{1}{e}\right)\OPT(T) - \mathcal{O}\left(k \sqrt{T \ln(T)} + k^2 + d_{\max}r \right).
\end{align*}
\end{restatable}

In the above bound, the additive loss corresponds to the regret with respect to $\left(1 - \frac{1}{e}\right)\OPT(T)$. Interestingly, our regret bound is very close (even in constant factors) to the information-theoretically optimal bound provided in \cite{KWAEE14} for the non-blocking setting. In fact, except for the small additive $\mathcal{O}(d_{\max}r)$ term, the regret bound in~\cite{KWAEE14} is the same as ours, if we replace the number of arms $k$ with $\sqrt{k\cdot r}$. Intuitively, this is due to the fact that our algorithm must learn the complete order of mean rewards, as opposed to the non-blocking setting where learning the maximum expected reward independent set in hindsight is sufficient for eliminating the regret.

All the omitted proofs of our results have been moved to the Appendix. We refer the reader to Appendix \ref{appendix:notation} for technical notation.
\section{Preliminaries on Matroids and Submodular Functions}\label{sec:definition}

\paragraph{Continuous extensions and the correlation gap of submodular functions.}

Consider any set function $f : 2^{\A} \rightarrow \mathbb{R}_{\geq 0}$ over a ground set $\A$. Recall that $f$ is submodular, if $\forall S,T \subseteq \A$ we have $f(S \cup T) + f(S \cap T) \leq f(S) + f(T)$. For any point $\x \in [0,1]^k$, we denote by $S \sim \x$ the random set $S \subseteq \A$, such that $\Pro{i \in S} = x_i$. We consider two canonical continuous extensions of a set function: 

\begin{definition}[Continuous extensions]
For any set function $f$ the {\em multi-linear extension} is
\begin{align*}
    F(\x) = \Ex{S \sim \x}{f(S)} = \sum_{S \subseteq \A} f(S) \prod_{i \in S} x_i \prod_{i \notin S} (1-x_i).
\end{align*}
Moreover, the {\em concave closure} is defined as
\begin{align*}
    f^+(\x) = \max_{\alpha}\{\sum_{S \subseteq \A} \alpha_S f(S)~|~\sum_{S \subseteq \A} \alpha_S \bm{1}_S = \x, \sum_{S \subseteq \A} \alpha_S = 1, \alpha \succeq 0\}.
\end{align*}
\end{definition}

\begin{lemma}[Correlation gap \cite{CCPV07}] \label{lem:correlationgap}
Let $f: 2^k \rightarrow \mathbb{R}_{\geq 0}$ be a monotone (non-decreasing) submodular function. Then for any point $\x \in [0,1]^k$, we have
$
F(\x) \leq f^+(\x) \leq \left(1 - \frac{1}{e}\right)^{-1} F(\x).
$
\end{lemma}

\paragraph{Matroid polytope and the weighted rank function.} Consider a matroid $\M = (\A, \I)$, where $\A$ is the {\em ground set} and $\I$ is the family of {\em independent sets}
\footnote{Any subset of the ground set $\A$ that is not independent is called {\em dependent}. Any maximal independent set of a matroid, namely, a set $B \in \I$ such that for every $e \in \A \backslash B$, the set $B \cup \{e\}$ is dependent, is called a {\em basis}. Any minimal dependent set, that is, a set $C \notin \I$ such that for each $e \in C$ it holds $C\backslash \{e\} \in \I$ is called a {\em circuit}.}. Recall that in any matroid, the family $\I$ satisfies the following two properties: (i) Every subset of an independent set (including the empty set) is an independent set, namely, if $S' \subset S \subseteq \A$ and $S \in \I$, then $S' \in \I$ ({\em hereditary property}). (ii) Let $S, S' \subseteq \A$ be two independent sets with $|S| < |S'|$, then there exists some $e \in S' \backslash S$ such that $S \cup \{e\} \in \I$ ({\em augmentation property}). See \cite{schrijver03, oxley06} for more details on matroids.

We assume that access to $\M$ is given through an {\em independence oracle} \cite{hausmann81, robinson80}, namely, a black-box routine that, given a set $S \subseteq \A$, answers whether $S$ is an independent set of $\M$. 
For any set $R \subset \A$ we define the {\em restriction} of $\M$ to $R$, denoted by $\M | R$, to be the matroid $\M | R = (R, \{I \in \I~|~I \subseteq R\})$. Every matroid $\M$ is associated with a {\em rank} function\footnote{The rank is {\em monotone} non-decreasing, {\em submodular}, and satisfies $\rk(S) \leq |S|$, $\forall S \subseteq \A$ (see \cite{oxley06}).} $\rk: 2^{\A} \rightarrow \mathbb{N}$, such that for any $S\subseteq \A$, $\rk(S)$ denotes the maximum size of an independent set contained in $S$. Let $\x(S) = \sum_{e \in S} x_e$ for some vector $\x \in \mathbb{R}^k$. For any matroid $\M = (\A,\I)$, the {\em matroid polytope} is defined as

$$\mathcal{P}(\M) \equiv \left\{ \x(S) \leq \rk(S), \forall S \in 2^{\A}, S \neq \emptyset \text{ and } \x \succeq 0 \right\}.
$$
It can be proved \cite{schrijver03} that the above polytope is the convex hull of the indicator vectors of all independent sets. This fact immediately leads to the following lemma: 

\begin{lemma}\label{lem:characteristic}
For any matroid $\M=(\A, \I)$ and point $\x \in \mathcal{P}(\M)$, there exists a collection of $k = |\A|$ independent sets $\I(\x) = \{ I_1, \dots, I_k\} \subseteq \I$ and a probability distribution over $\I(\x)$ such that $\Prob{I \sim \I(\x)}{i \in I} = x_i$, i.e., an element $i$ belongs to a sampled set with marginal probability equal to $x_i$.
\end{lemma}
Given any non-negative linear {\em weight} vector $\w \in \mathbb{R}^k_{\geq 0}$, the problem of computing a maximum weight independent set can be solved optimally by the standard greedy algorithm: Starting from the empty set $S = \emptyset$, add each ground element $e \in \A$ to the set $S$ in a non-increasing order of weights, as long as the set $S \cup \{e\}$ does not contain a circuit. Given a matroid $\M=(\A,\I)$ and a weight vector $\w$, the function $f_{\M,\w}(S) = \max_{I \in \I, I \subseteq S}\{\w(I)\}$ is called the {\em weighted rank function} of $\M$ and returns the weight of the maximum independent set of the restriction $\M|S$.

\begin{lemma}[Weighted rank function \cite{CCPV07}] \label{lem:weightedrank}
For any matroid $\M$ and non-negative weight vector $\w$, the function $f_{\M, \w}(S) = \max_{I \in \I, I \subseteq S}\{\w(I)\}$ is monotone (non-decreasing) submodular.
\end{lemma}

%\begin{lemma}[Correlation Gap] %\label{lem:correlationggap}
%Let $f$ be a monotone (non-decreasing) submodular function. For the sets $S$ and $T_1, \dots, T_k$ constructed as described above, we have:
%\begin{align*}
%    \Ex{S \sim \x}{f(S)} \geq \left(1 - \frac{1}{e}\right) \Ex{T \sim \I(\x)}{f(T)}.
%\end{align*}
%\end{lemma}

\section{Recurrent Submodular Welfare}
Let $f(S): 2^{\A} \rightarrow \mathbb{R}_{\geq 0}$ be a monotone submodular function over a universe $\A$ of $k$ elements, such that $f(\emptyset) = 0$. In the {\em blocking} setting, each element $i \in \A$ is associated with a known deterministic {\em delay} $d_i \in \mathbb{N}_{>0}$, such that once the arm is played at some round $t$, it becomes unavailable for the next $d_i-1$ rounds, namely, in the interval $\{t, \dots, t+d_i-1\}$. At each round $t \in [T]$, the player chooses a subset $\A_t$ of available (i.e., non-blocked) elements and collects a reward $f(\A_t)$. The goal is to maximize the total reward collected, i.e., $\sum_{t \in [T]} f(\A_t)$, within an unknown time horizon $T$. 

Before we present our algorithm, we provide ``bad'' instances for two natural approaches to \rsm.

\begin{remark} \label{rem:greedy}
The greedy approach of choosing $\A_t$ to be the set of all available elements at round $t \in [T]$ can be as bad as a $\frac{1}{k}$-approximation. In order to see that, consider the monotone (budget-additive) submodular function $f(S) = \min\{|S|, 1\}$. Let $k$ be the number of elements with delay $d_i = k$ for each $i \in \A$. Assuming an infinite time horizon, the optimal strategy collects an average reward of $1$, simply by choosing one element at a time in a round-robin manner. However, the average reward of the greedy approach in this case is $\frac{1}{k}$.
\end{remark}

\begin{remark}
The independent randomized sampling approach of adding each arm $i$ to $\A_t$ independently with probability $\frac{1}{d_i}$, if available, can be as bad as a $(1 - \frac{1}{\sqrt{e}} )$-approximation. Consider the same setting as in Remark \ref{rem:greedy}, where for $T \to \infty$ the optimal average reward is $1$. However, the average expected reward of the independent randomized sampling strategy is $1 - (1 - p)^k$, where $p = \frac{1}{2k-1}$ is the probability that each element is selected at each round (in stationarity). For $k \to \infty$, we have that $1 - (1 - p)^k \to 1- e^{-\frac{1}{2}} \approx 0.393$.

\end{remark}
We provide an efficient randomized $\left(1-\frac{1}{e}\right)$-approximation algorithm for \rsm. Informally, the algorithm starts by considering, for each element $i \in \A$, a sequence of rational numbers of the form $\{t\cdot \frac{1}{d_i}\}_{t \in [T]}$. Then, these sequences are {\em interleaved} by randomly adding an {\em offset} $r_i$, drawn uniformly at random from $[0,1]$, for each $i \in \A$ to the corresponding sequence. At every round $t \in [T]$, the algorithm chooses a set $\A_t$, consisting only of elements for which the (perturbed) interval $L_{i,t} = [t\cdot \frac{1}{d_i}+ r_i, (t+1)\cdot \frac{1}{d_i}+ r_i )$ contains an integer.

\begin{algorithm}[\is (\IS)]
For each element $i \in \A$, let $r_i \sim U[0,1]$ be a random {\em offset} drawn uniformly from $[0,1]$. 
At every round $t = 1, 2, \dots$,  let $\A_t \subseteq \A$ be the subset of elements such that for any $i \in \A_t$, the interval $L_{i,t} = [t\cdot \frac{1}{d_i} + r_i, (t+1) \cdot \frac{1}{d_i} + r_i)$ contains an integer. Choose the elements $\A_t$ and collect the reward $f(\A_t)$.
\end{algorithm}

\subsection{Correctness and approximation guarantee.} 
We first show the algorithm is correct, namely, that the elements chosen at each round respect the blocking constraints. The correctness is established by the following simple observation:

\begin{restatable}{fact}{restatefactalwaysavailable}\label{inter:fact:alwaysavailable}
At any $t \in [T]$, all the elements in $\A_t$ are available (i.e., not blocked).
\end{restatable}

In order to prove the competitive guarantee of our algorithm, we first construct a convex programming (CP)-based (approximate) upper bound on the optimal reward. Although our algorithm never computes an optimal solution to this CP, this step allows us to prove our guarantee, leveraging results on the correlation gap of submodular functions. For $\bm{d}^{-1} \in \mathbb{R}^k$ such that $(\bm{d}^{-1})_i = \frac{1}{d_i}, \forall i \in [k]$, consider the following formulation based on the concave closure $f^+$ of $f$:
\begin{align}
\maximize_{\z \in \mathbb{R}^k}~~ T \cdot f^+(\z)~~\textbf{s.t.}~~ \bm{0} \preceq \z \preceq \bm{d}^{-1}. \tag{\textbf{CP}} \label{cp:CP}
\end{align}

In \eqref{cp:CP}, each variable $z_{i}$ can be thought of as the fraction of rounds where element $i\in \A$ is chosen. Intuitively, the constraints indicate the fact that, due to the blocking, each element $i \in \A$ can be played at most once every $d_i$ steps. 
In order to derive \eqref{cp:CP}, we start from a non-convex integer program (IP) with 0-1 variables $\{x_{i,t}\}_{i \in \A, t \in [T]}$, each indicating whether element $i \in \A$ is used at round $t \in [T]$. The objective is to maximize $\sum_{t \in [T]} \sum_{S \subseteq \A} f(S) \prod_{i \in S} x_{i,t} \prod_{i \notin S}(1 - x_{i,t})$ subject to natural blocking constraints. For integral solutions, the above objective is equivalent to $\sum_{t \in [T]} f^+(\x_t)$ (where $(\x_t)_i = x_{i,t}$) and, thus, the above relaxation is simply the result of averaging over time the variables and constraints of this IP. By using the concavity of $f^+$, we are able to show that \eqref{cp:CP} yields an (approximate) upper bound on the optimal solution of \rsm, while the approximation becomes exact as $T$ increases.

\begin{restatable}{lemma}{restateStructuralCP}\label{lem:structural:CP}
Let $\Rew^{CP}(T)$ be the optimal solution to \eqref{cp:CP} and $\OPT(T)$ be the optimal solution over $T$ rounds. We have
$
\Rew^{CP}(T) \geq \OPT(T) - \mathcal{O}(d_{\max} f(\A)),
$ where $d_{\max} = \max_{i \in \A}\{d_i\}$.
\end{restatable}

\begin{remark}
By replacing $f^+(\z)$ in \eqref{cp:CP} with the multi-linear extension $F(\z)$, the formulation no longer yields an upper bound on the optimal reward (not even asymptotically). Indeed, consider a function $f$ over a ground set $\A=\{1,2\}$ with $d_1 = d_2 = 2$, such that $f(\emptyset) = 0$, $f(\{1\}) = f(\{2\}) = 2$ and $f(\{1,2\}) = 3$. For $T \to \infty$, the optimal average reward is $2$, simply by choosing the two elements interchangeably. However, the formulation based on $F(\z)$ in that case would be to maximize $2z_1(1-z_2) + 2z_2(1-z_1) + 3 z_1 z_2$ subject to $z_1,z_2 \leq \frac{1}{2}$, which has a global maximum of $\frac{7}{4} < 2$.
\end{remark}

Before we complete the proof of our first main result, we first compute the probability that $i \in \A_t$, i.e., an element $i \in \A$ is sampled at round $t \in [T]$:

\begin{restatable}{fact}{restatefactsampling}\label{inter:fact:sampling}
For any $i \in \A$ and $t \in [T]$, we have
$\Pro{i \in \A_t} = \Pro{L_{i,t} \cap \mathbb{N} \neq \emptyset } = \frac{1}{d_i}.
$
\end{restatable}

\noindent{\em Proof of Theorem \ref{thm:interleavedSubmodular}.} 
Let us denote by $S \sim {\bf p}$ with ${\bf p} \in [0,1]^k$ the random set $S \subseteq \A$, where each element $i$ participates in $S$ independently with probability equal to $p_i$. 
By Fact~\ref{inter:fact:sampling} and due to the randomness of the offsets $\{r_i\}_{i \in \A}$, we have that $\A_t \sim {\bf d}^{-1}$ for each $t \in [T]$. Let $\z^*$ be an optimal solution to \eqref{cp:CP}. By monotonicity of $f$ and the fact that $\z^* \preceq \bm{d}^{-1}$, for the expected value of $f(\A_t)$ at any round $t \in [T]$, we know that $\Ex{\A_t \sim \bm{d}^{-1}}{f(\A_t)} \geq \Ex{\A_t \sim \z^*}{f(\A_t)}$. Moreover, by definition of the multi-linear extension, we have that $\Ex{\A_t \sim \z^*}{f(\A_t)} = F(\z^*)$, while by Lemma~\ref{lem:correlationgap} (the correlation gap of submodular functions), we have that, $F(\z) \geq \left(1 - \frac{1}{e}\right)f^+(\z)$ for any vector $\z \in [0,1]^k$. By combining the above facts, we can see that
\begin{align*}
\Rew^{IS}(T) = \sum_{t \in [T]} \Ex{\A_t \sim \bm{d}^{-1}}{f(\A_t)} \geq 
\sum_{t \in [T]} F(\z^*) \geq \left(1 - \frac{1}{e}\right)T\cdot f^+(\z^*) = \left(1 - \frac{1}{e}\right) \Rew^{CP}(T).
\end{align*}
Therefore, by Lemma~\ref{lem:structural:CP}, we can conclude that $\Rew^{IS}(T) \geq \left(1 - \frac{1}{e}\right)\OPT(T) - \mathcal{O}(d_{\max} f(\A))$.
\qed
\newline

In Appendix \ref{appendix:hardness}, we provide a $\left(1-\frac{1}{e}\right)$-hardness result for \rsm, thus proving that the guarantee of Theorem~\ref{thm:interleavedSubmodular} is asymptotically tight. This result, which holds even for the special case where $d_{\max} = o(T)$ (that is when the delays are significantly smaller than the time horizon), is proved via a reduction from the SWM problem with identical utilities, in a way that the constructed \rsm instance accepts w.l.o.g. solutions of a simple periodic structure.

\begin{restatable}{theorem}{restateSubmodularHardness}\label{thm:submodular:hardness}
For any $\epsilon>0$, there exists no polynomial-time $\left(1-\frac{1}{e} + \epsilon \right)$-approximation algorithm for the \rsm problem, unless ${\bf P}={\bf NP}$, even in the special case where $d_{\max} = o(T)$.
\end{restatable}
\section{Matroid Blocking Bandits}
\label{section:mbb}
Let $\A$ be a set of $k$ arms and $T$ be an unknown time horizon. At any round $t \in [T]$ and for each $i \in \A$ a reward $X_{i,t}$ is drawn independently from an unknown distribution of mean $\mu_{i}$ and bounded support in $[0,1]$. Let $d_i \in \mathbb{N}_{>0}$ be the known determinisitc delay of each arm $i \in \A$, and $d_{\max} = \max_{i \in \A}\{d_i\}$. At any round $t \in [T]$, the player pulls any subset $\A_t$ of the available (i.e., non-blocked) arms, as long as it forms an {independent} set of a given {matroid} $\M = \left(\A,\I\right)$. The player only observes the realized reward of each arm she plays and collects their sum. The goal is to maximize the {\em expected cumulative reward} collected within $T$ rounds, denoted by
$\Rew^{IG}(T) = \Ex{ }{\sum_{t \in [T]} \sum_{i \in \A} X_{i,t} \event{i \in \A_t}}$. 

\subsection{The full-information setting}

The following algorithm is the implementation of \IS in the special case of the full-information \mbb setting, where the mean rewards $\{\mu_{i}\}_{i \in \A}$ are known to the player a priori:

\begin{algorithm}[\ig (\IG)]
For each arm $i \in \A$, let $r_i \sim U[0,1]$ be a random {\em offset} drawn uniformly from $[0,1]$. 
At every round $t = 1, 2, \dots$,  let $\G_t \subseteq \A$ be the subset of arms $i \in \A$, such that the interval $L_{i,t} = [t\cdot \frac{1}{d_i} + r_i, (t+1) \cdot \frac{1}{d_i} + r_i)$ contains an integer. Greedily compute a maximum independent set $\A_t$ of $\M|\G_t$ with respect to $\{\mu_i\}_{i \in \G_t}$ and play these arms.
\end{algorithm}

The correctness of the above algorithm follows directly by Fact~\ref{inter:fact:alwaysavailable} (that is, the sampled arms are never blocked) and by the fact that \IG always plays an independent set of $\M$, i.e., $\A_t \in \I_{\M|\G_t} \subseteq \I$. The approximation guarantee of \IG follows immediately as a special case of Theorem~\ref{thm:interleavedSubmodular}. Indeed, notice that: (i) The reward realizations do not affect the choices of \IG or any optimal algorithm maximizing the total expected reward. Thus, each realization $X_{i,t}$ can be replaced w.l.o.g. by its expected value $\mu_i$. (ii) The value of the greedily computed maximum independent set in $\M|\G_t$ corresponds to the weighted rank function $f_{\M, \mu}(\G_t)$ which, by Lemma~\ref{lem:weightedrank}, is monotone submodular.

\begin{restatable}{theorem}{restateinterleavedGreedy}\label{thm:interleavedGreedy}
The expected reward collected by \ig for $T$ rounds, $\Rew^{IG}(T)$, is at least
$
\left(1 - \frac{1}{e} \right) \OPT(T) - \mathcal{O}(d_{\max} \rk(\M))
$, where $\OPT(T)$ is the optimal expected reward.
\end{restatable}

As a point of interest, in Appendix \ref{appendix:mbb:fullinformation} we provide an alternative proof of the above theorem, which, instead of the concave closure of the weighted rank, now relies on the following (approximate) LP upper bound, based on the matroid polytope. For any set $S \subseteq \A$, let $\z(S) = \sum_{i \in S} z_i$.
\begin{align}
&\maximize_{\z \in \mathbb{R}^k} ~T \cdot \sum_{i \in \A} \mu_{i} z_{i} \tag{\textbf{LP}} \quad 
\textbf{s.t. } \z(S) \leq \rk(S)~ \forall S\subseteq \A
\textbf{, } \bm{0} \preceq \z \preceq \bm{d}^{-1}.
\end{align}

\begin{remark}
The analysis of $\IG$ is tight for rank-1 matroids. Indeed, consider $k$ arms, each of delay $k$ and deterministic reward equal to $1$. For $T \to \infty$, the optimal average reward is equal to $1$, simply by playing the arms in a round-robin manner. However, the probability that at least one arm is sampled at some round $t$ is equal to $\sum^k_{i = 1} {k \choose i} \left(\frac{1}{k}\right)^i \left(1 - \frac{1}{k}\right)^{k - i} = 1 - \left(1 - \frac{1}{k}\right)^k \to 1 - \frac{1}{e}$ as $k \to \infty$.
\end{remark}

\subsection{The bandit setting and regret analysis}

In the setting where the mean rewards are initially unknown, we develop a UCB-based bandit algorithm, \ucb (\UCB). The algorithm is identical to \IG, except for the greedy computation of the maximum independent set over the sampled arms, which is now performed using estimates. Specifically, the algorithm maintains for every $i \in \A$, $t \in [T]$ the following upper estimate of $\mu_i$:
\vspace{-0.5em}
\begin{align*}
\bar{\mu}_{i,t} = \hat{\mu}_{i,T_{i}(t)} + c_{i,t}\text{  with  } c_{i,t} = \sqrt{\frac{2 \ln{(t)}}{T_{i}(t)}},
\end{align*}
where $T_{i}(t)$ denotes the number of times arm $i$ has been played at the beginning of round $t$ and $\hat{\mu}_{i,T_{i}(t)}$ denotes the empirical average of the $T_{i}(t)$ i.i.d. samples from its reward distribution. The term $c_{i,t}$ is the {\em confidence length} around $\hat{\mu}_{i,T_{i}(t)}$ that guarantees $\bar{\mu}_{i,t}$  lies in $[\mu_i, \mu_i + 2c_{i,t}]$ with high probability. Note that all the above quantities are random variables depending on the random offsets and the observed reward realizations.

We are interested in upper bounding the $\alpha$-regret, for $\alpha = 1- \frac{1}{e}$, namely, the difference between $\alpha \OPT(T)$ and the expected reward collected by \UCB. Due to the complex time dynamics, characterizing the optimal expected reward as a function of the instance is hard. However, using Theorem~\ref{thm:interleavedGreedy} we can upper bound $\alpha \OPT(T)$ by the expected reward collected by \IG, thus giving:
\begin{align} \label{eq:regret:twoalgos}
\alpha \OPT(T) - \Rew^{UCB}(T)  \leq \Rew^{IG}(T) - \Rew^{UCB}(T) + \mathcal{O}(d_{\max} \cdot \rk(\M)).
\end{align}
By the above inequality, it becomes clear that in order to upper bound the regret, it suffices to bound the difference between the expected reward collected by \IG and \UCB. This difference not only depends on the reward realizations (through the UCB estimates), but also on the trajectory of sampled arms in each algorithm, which is itself a function of the random offsets. However, by construction of our interleaved scheduling scheme, these offsets are sampled at the initialization phase of each algorithm and are identically distributed. Thus, the trajectories of sampled arms in the two algorithms exhibit a coupled evolution. This allows us to analyse the regret ``pointwise'', under the assumption that the sequences of sampled arms are identical throughout the time horizon. To make this idea precise, let $\off^{\pi} \in [0,1]^k$ be the random offsets used and let $\{\G^{\pi}_t(\off^{\pi})\}_{t \in [T]}$ be the sequence of sampled arms by algorithm $\pi \in \{\IG, \UCB\}$. Using (henceforth) $\mathcal{R}$ to denote the randomness due to the reward realizations of the arms, the next lemma gives our pointwise regret bound.

\begin{restatable}{lemma}{restateRegretCoupling}
\label{lem:regret:coupling} Let $\bar{\mu}_t(S) = \sum_{i \in S} \bar{\mu}_{i,t}$ and $\mu(S) = \sum_{i \in S} \mu_i$. We have
\begin{align*}
    \Rew^{\IG}(T) - \Rew^{\UCB}(T) =  \Ex{\off \sim U[0,1]^k, \mathcal{R}}{\sum_{t \in [T]}\left( \max_{S \subseteq \G_t(\off), S \in \I} \{ \mu\left(S\right)\} - \mu\left(\arg\max_{S \subseteq \G_t(\off), S \in \I}\{\bar\mu_t(S)\}\right)\right)}.
\end{align*}
\end{restatable}

Thus w.l.o.g., we focus on the case where the sequences of sampled arms are identical. Let  $\mathcal{E}_{\off}$ denote the event that both algorithms, \IG and \UCB, sample the same offset vector $\off$, namely, $\off^{\IG} = \off^{\UCB} = \off$. Assuming that $\mathcal{E}_{\off}$ holds for some $\off \in [0,1]^k$, let $\{\G_t\}_{t \in [T]} = \{\G_t(\off)\}_{t \in [T]}$ be the sequence of sampled arms, common in both algorithms. Clearly, \UCB accumulates regret only when it plays independent sets of arms that are suboptimal w.r.t.\ the true means, i.e., when $\mu(\A^{\UCB}_t) < \mu(\A^{\IG}_t)$ for some $t \in [T]$. We assume w.l.o.g.\ that the arms are indexed in decreasing order of mean rewards and that these mean rewards are distinct. We now formally define the gaps related to our analysis:

\begin{definition}[Gaps]\label{def:gaps}
For any subset $S \subseteq \A$ and reward vector $\nu \in \mathbb{R}^k$, we define 
\begin{align*}
\Delta_S(\nu) = \max_{I \in \I, I \subseteq S}\{\mu\left(I\right)\}  - \mu\left(\arg\max_{B \in \I, B \subseteq S}\{\nu\left(B\right)\}\right).
\end{align*}
Moreover, let $\Delta_{i,j} = \mu_i - \mu_{j}$ be the standard suboptimality gap between two arms $i,j \in \A$.
\end{definition}

By Lemma \ref{lem:regret:coupling} and assuming that the event $\mathcal{E}_{\off}$ holds for some $\off$, we are interested in bounding the expectation of $\sum_{t \in [T]} \Delta_{\G_t(\off)}(\bar\mu_t)$ w.r.t.\ the reward realizations. The next step is to decompose the suboptimality of \UCB by noticing that both algorithms play, at each round $t \in [T]$, a basis of $\M|\G_t$ and thus $|\A^{\IG}_t| = |\A^{\UCB}_t|$. We use the following fundamental property of matroids:

\begin{theorem}[Strong Basis Exchange, Corollary 39.12a in \cite{schrijver03}]\label{cor:basesexchange} Let $\M = (\A,\I)$ be a matroid and $I_1, I_2 \in \I$ be two independent sets such that $|I_1| = |I_2|$. Then, there exists a bijection $\sigma : I_1 \rightarrow I_2$, such that for any $i \in I_1$ the set $I_1 - i + \sigma(i)$ is an independent set of $\M$. 
\end{theorem}

Let $\sigma_t: \A^{\UCB}_t \rightarrow \A_t^{\IG}$ for each $t \in [T]$ be the bijection described in Theorem~\ref{cor:basesexchange} with respect to the sets $\A^{\UCB}_t$ and $\A^{\IG}_t$ and let $\sigma_t^{-1}$ be its inverse mapping. 
Note that in any bijection $\sigma_t$ and any $i \in \A^{\UCB}_t \cap \A^{\IG}_t$ we can assume w.l.o.g. that $\sigma_t(i) = i$. Notice, further, that under the event $\mathcal{E}_{\off}$, the bijections $\{\sigma_t\}_{t \in [T]}$ are still random variables that depend on the observed realizations. 

\begin{restatable}{lemma}{restateRegretSub} \label{lem:regret:suboptdecomp} Under the event $\mathcal{E}_{\off}$ and at any time $t \in [T]$, we have 
$\Delta_{\G_t}(\bar{\mu}_t) = \sum_{i \in \A^{\IG}_t} \Delta_{i , \sigma^{-1}_{t}(i)}$.
\end{restatable}
Conditioned on the fact that both algorithms operate on the same sequence $\{\G_t\}_{t \in [T]}$ of sampled arms, Lemma \ref{lem:regret:suboptdecomp} allows us to decompose the suboptimality gap $\Delta_{\G_t}(\bar{\mu}_t)$ of each round $t \in [T]$, into simpler gaps of the form $\Delta_{i,j}$ between any arms $i \in \A^{\IG}_t$ and $j \in \A^{\UCB}_t$ that are perfectly matched according to the bijection $\sigma_t$, namely, $\sigma_t(j) = i$. Assuming that the event $\{\sigma_t(j) = i\}$ directly implies that $i \in A^{\IG}_t$ and $j \in A^{\UCB}_t$, we can further upper bound the regret as
\begin{align*}
\sum_{t \in [T]} \Delta_{\G_t}(\bar{\mu}_t) = \sum_{t \in [T]} \sum_{i \in \A^{\IG}_t} \Delta_{i , \sigma^{-1}_{t}(i)} \leq \sum_{t \in [T]} \sum_{i \in \A^{\IG}_t} \sum_{j \in \A, \Delta_{i,j} > 0} \Delta_{i,j} \event{\sigma_t(j)=i}.
\end{align*}

The above inequality allows us to study the regret attributed to each arm independently, using more standard arguments for UCB-based algorithms in combination with Theorem~\ref{cor:basesexchange}. Specifically, for every pair of arms $i,j \in \A$ with $i < j$ (thus, $\Delta_{i,j} > 0$), we define a threshold $\ell_{i,j}$ with the following key-property: After \UCB ``exchanges'' arm $j$ for arm $i = \sigma_t(j)$ more than $\ell_{i,j}$ times, due to insufficient exploration, then it has collected enough samples to infer that $\mu_j < \mu_i$ with high probability. 
\begin{restatable}{lemma}{restateRegretTechnical} \label{lem:regret:technical}
Let $\ell_{i,j} = \bigg\lfloor \frac{8 \ln(T)}{\Delta^2_{i,j}}\bigg\rfloor$ for any $i<j$. Under event $\mathcal{E}_{\off}$ and for any arm $j>1$, we have
\vspace{-1em}
\begin{align}
&\sum_{t \in [T]} \sum_{i < j} \Delta_{i , j} \event{\sigma_t(j)=i, T_j(t) \leq \ell_{i,j}} \leq \frac{16}{\Delta_{j-1,j}} \ln( T) &\quad\text{(Under-sampled regret)} \label{inq:reg:under}\\
& \Ex{\mathcal{R}}{\sum_{t \in [T]} \sum_{i < j} \Delta_{i , j} \event{\sigma_t(j)=i, T_j(t) > \ell_{i,j}}} \leq \frac{\pi^2}{3}\sum^{j-1}_{i=1} \Delta_{i,j} &\quad\text{(Sufficiently sampled regret)} \label{inq:reg:sufficiently}
\end{align}
\end{restatable}

\begin{proof}[Proof sketch of Theorem \ref{thm:approxregret}] 
By inequality \eqref{eq:regret:twoalgos} and Lemma \ref{lem:regret:coupling}, in order to bound the regret of \UCB, it suffices to upper bound the difference between $\Rew^{\IG}(T)$ and $\Rew^{\UCB}(T)$, conditioned on the fact that both algorithms use exactly the same offset vector $\off$ and, thus, they operate on the exact same sequence of sampled arms, denoted by $\{\G_t\}_{t \in [T]}$. By construction, \IG plays at any round $t \in [T]$ a basis of $\M|\G_t$ of maximum expected reward, while \UCB plays a basis of $\M|\G_t$ that is maximum with respect to the estimates $\{\bar{\mu}_{i,t}\}_{i \in \A}$. By Theorem~\ref{cor:basesexchange}, we can consider a perfect matching between exchangeable arms of $\A^{\IG}_t$ and $\A^{\UCB}_t$ and, thus, to decompose the regret into suboptimality gaps between individual arms. Then, using Lemma~\ref{lem:regret:technical}, we can upper bound on the expected regret due to the fact that \UCB erroneously plays arm $j$ instead of arm $i$, when $\Delta_{i,j} > 0$. The above analysis culminates in a regret bound that is a function of $\{\Delta_{i,j}\}_{i,j\in \A}$. In order to derive a gap-independent regret bound, we partition the gaps into ``small'' and ``large'' and notice that any pair of arms $i,j \in \A$ with $\Delta_{i,j} < \Theta(\sqrt{\frac{\ln(T)}{T}})$ cannot contribute more than $\sqrt{T\ln(T)}$ loss in the regret.
\end{proof}

\section*{Conclusion and Further Directions}
We explore the effect of action-reward dependencies in the combinatorial MAB setting by introducing and studying the \mbb problem. After relating the problem to \rsm, we provide a $\left(1-\frac{1}{e}\right)$-approximation for its full-information case, based on the technique of interleaved scheduling. Importantly, our technique is oblivious to the reward distributions of the arms-- a fact that allows us to provide regret bounds of optimal dependence in $T$, when these distributions are initially unknown.

Our work leaves behind numerous interesting questions. By exhaustive search over $\mathcal{O}(1)$-periodic schedules, one can construct a PTAS for the (asymptotic) \mbb problem, assuming {\em constant} $\rk(\M)$ and $\{d_i\}_{i \in [k]}$. It remains an open question, however, whether the $\left(1-\frac{1}{e}\right)$-approximation is the best possible in general. We remark that the hardness of \mbb cannot solely rely on an argument similar to Theorem \ref{thm:submodular:hardness}, since the welfare maximization problem for the class of {\em gross substitutes}, which includes weighted matroid rank functions, is easy \cite{Leme17}. Another interesting direction would be to study natural extensions of the \rsm problem, when additional constraints (knapsack, matroid etc.) are imposed on top of blocking, or when the submodularity assumption is relaxed (see, e.g., \cite{Feige06}).

%{\color{gray}
%Finally, we believe that the study of (combinatorial) bandits under natural temporal correlations between the reward distributions and the player's actions (other than blocking) remains a wide source of interesting algorithmic problems.
%}

\section*{Acknowledgements}
The authors would like to thank an anonymous reviewer of a previous version of this work for an unusually thoughtful and helpful review, which aided us in improving the document — in particular, for pointing out the idea of correlated rounding. Further, the authors would like to thank Jannik Matuschke for noticing that the weighted matroid rank function falls into the class of gross substitutes.

\bibliographystyle{plainurl}
\bibliography{ref.bib}

\newpage
\appendix
\section{Technical Notation} \label{appendix:notation}

For any event $\mathcal{E}$, we denote by $\event{\mathcal{E}} \in \{0,1\}$ the indicator variable such that $\event{\mathcal{E}} = 1$, if $\mathcal{E}$ occurs, and $\event{\mathcal{E}} = 0$, otherwise. For any non-negative integer $n \in \mathbb{N}$, we define $[n] = \{1,2, \dots, n\}$. For any vector $\mu \in \mathbb{R}^k$ and set $S \subseteq [k]$, we define $\mu(S) = \sum_{i \in S} \mu_i$. Moreover, we use the notation $t \in [a, b]$ (for $a \leq b$) for some time index $t$, in place of $t \in [T] \cap [a, \ldots , b\}$. Unless otherwise noted, we use the indices $i$, $j$ or $i'$ to refer to arms and $t$, $t'$ or $\tau$ to refer to time. Let $\A^{\pi}_t \in \I$ be the set of arms played by some algorithm $\pi \in \{\IS, \IG, \UCB\}$ at time $t$. Unless otherwise noted, all expectations are taken over the randomness of the offsets $\{r_i\}_{i \in [k]}$ and the reward realizations.

\section{Concentration inequalities}
\begin{theorem}[Hoeffding's Inequality \cite{Hoeffding}]\label{appendix:concentration:hoeffding}
Let $X_1, \dots, X_n$ be independent identically distributed random variables with common support in $[0,1]$ and mean $\mu$. Let $Y = X_1 + \dots + X_n$. Then for any $\delta \geq 0$,
\begin{align*}
    \Pro{Y-n\mu \geq \delta} \leq e^{-2\delta^2/n} \text{   and   }\Pro{Y-n\mu \leq -\delta} \leq e^{-2\delta^2/n}.
\end{align*}
\end{theorem}
\section{Recurrent Submodular Welfare: Omitted Proofs}
\subsection{Correctness and approximation guarantee}

\restatefactalwaysavailable*
\begin{proof}
Recall that at any round $t \in [T]$, the algorithm only chooses a subset $\A_t$ of the elements. Consider any element $i \in \A$ such that $i \in \A_t$ for some $t \in [T]$. By definition of $\A_t$, the interval $L_{i,t} = [t\cdot \frac{1}{d_i} + r_i, (t+1) \cdot \frac{1}{d_i} + r_i)$ contains an integer. It is not hard to see that, in that case, none of the intervals $L_{i,t'}$ for $t' \in [t-d_{i}+1, d_i - 1]$ can contain an integer. Therefore, the last time element $i$ has been chosen must be before $t- d_i$, which implies feasibility with respect to the blocking constraints.
\end{proof}

\restatefactsampling*

\begin{proof}
For any fixed $i \in \A$ and $t \in [T]$, because of the fact that $\frac{1}{d_i} \leq 1$ and $r_i \in [0,1]$, the interval $L_{i,t} = [t\cdot \frac{1}{d_i} + r_i, (t+1) \cdot \frac{1}{d_i} + r_i)$ clearly contains at most one integral point. The event that $\{[t\cdot \frac{1}{d_i} + r_i, (t+1) \cdot \frac{1}{d_i} + r_i)\cap \mathbb{N} \neq \emptyset \}$ is equivalent to the event that a continuous window of size equal to $\frac{1}{d_i}$ starting from the (real) point $t\cdot \frac{1}{d_i} + r_i$ contains an integer. For $r_i$ ranging in $[0,1]$, the starting point of the interval lies between $t\cdot \frac{1}{d_i}$ and $t\cdot \frac{1}{d_i} + 1$. It is not hard to see that fraction of possible realizations of $r_i$ such that the window contains an integer equals its size. The fact follows since for any $i \in \A$, the window has size $\frac{1}{d_i}$ and the offset $r_i$ is sampled uniformly at random from $[0,1]$.
\end{proof}

\restateStructuralCP*
\begin{proof}
In order to prove the lemma, we first construct an (non-convex) IP upper bound on the optimal expected reward over $T$ rounds, based on the multi-linear extension of $f$.

\begin{align}
\textbf{maximize:}& \sum_{i \in [T]} \sum_{S \subseteq \A} f(S) \prod_{i \in S} x_{i,t} \prod_{i \notin S} (1-x_{i,t}) \tag{\textbf{MP}} \label{mp:MP}\\
\textbf{s.t.}& \sum_{t' \in [t,t+d_i-1]} x_{i,t'} \leq 1, \forall i \in \A, \forall t \in [T] \label{mp:window} \\
\qquad &\x_{t} \in \{0,1\}^k, \forall t \in [T] \notag
\end{align}

In the formulation \eqref{mp:MP}, each variable $x_{i,t}$ can be thought of as the 0-1 indicator of playing arm $i\in \A$ at time $t \in [T]$. Intuitively, constraints \eqref{mp:window} of \eqref{mp:MP} indicate the fact that, due to blocking constraints, each arm $i \in \A$ can be played at most once every $d_i$ steps. Clearly, any optimal solution to \rsm can be mapped onto the above formulation and, thus, the optimal solution of \eqref{mp:MP} provides an upper bound on $\OPT(T)$.

Let $\x_t \in \{0,1\}^k$ for each $t \in [T]$ be a vector such that $(\x_t)_i = x_{i,t}$. Notice that for any integral $\x \in \{0,1\}^k$, the multi-linear extension is equal to the concave closure of any set function $f$, that is, $f^+(\x) = F(\x)$. Therefore, \eqref{mp:MP} remains an upper bound, even if we replace its objective function with $g(\x_1, \dots, \x_T) = \sum_{t \in [T]} f^+(\x_t)$.

We now fix any optimal solution $\{x^*_{i,t}\}_{i\in \A, t \in [T]}$ to \eqref{mp:MP} under the objective $g(\x_1, \dots, \x_T) = \sum_{t \in [T]} f^+(\x_t)$. Let us define the variables $\{z'_{i}\}_{i \in \A}$, such that
\begin{align*}
    z'_i = \frac{1}{T} \sum_{t \in [T]} x^*_{i,t} \geq 0, \quad \forall i \in \A.
\end{align*}
In the above definition, each $z'_i$ is the fraction of time an element $i \in \A$ is chosen in an optimal solution. Let $\z' \in [0,1]^k$, such that $(\z')_i = z'_i$ $\forall i \in \A$.

By concavity of $f^+$, we have
\begin{align*}
g(\x^*_1, \dots, \x^*_T) = \sum_{t \in [T]} f^+(\x^*_t) = T \sum_{t \in [T]} \frac{1}{T} f^+(\x^*_t) \leq T f^+(\frac{1}{T}\sum_{t \in [T]}\x^*_t) = T f^+(\z'),
\end{align*}
where the inequality follows by the fact that $\z'$ can be thought of as a convex combination of $\{\x^*_1, \dots, \x^*_T\}$.

Moreover, for each $i \in \A$ and by averaging constraints \eqref{mp:window} of \eqref{mp:MP} over all $t \in [T]$, we can see that 
\begin{align*}
\frac{1}{T}\sum_{t \in [1,d_i-1]} t x^*_{i,t} + \frac{1}{T} \sum_{t \in [d_i,T]} d_i x^*_{i,t} \leq 1 \Leftrightarrow \frac{1}{T} \sum_{t \in [T]} d_i x^*_{i,t} \leq 1 + \frac{1}{T}\sum_{t \in [1,d_i-1]} (d_i - t) x^*_{i,t}. 
\end{align*}
Given the fact that $\sum_{t \in [1, d_i-1]} x^*_{i,j} \leq 1$, the above inequality immediately implies that
\begin{align*}
    z'_i \leq \frac{1}{d_i}\left(1 + \frac{d_i-1}{T} \right)\quad \forall i \in \A.
\end{align*}
Consider now the assignment $z_i = \left(1 + \frac{d_{\max}-1}{T} \right)^{-1} z'_i$, $\forall i \in \A$. For this assignment, we can easily verify that the constraints of \eqref{cp:CP} are trivially satisfied, since $0 \leq z_i \leq \frac{1}{d_i}$, $\forall i \in \A$.

Let $\z \in [0,1]^k$, such that $(\z)_i = z_i$ $\forall i \in \A$. By the above analysis, we can see that
\begin{align*}
    \z = \z' - \frac{d_{\max}-1}{T + d_{\max} -1} \z',
\end{align*}
where we use the fact that $\frac{1}{1+\beta} = 1 - \frac{\beta}{1+ \beta}$ for any $\beta \in \mathbb{R}$.
Finally, by concavity of $f^+$ we have
\begin{align*}
    f^+(\z) &= 
    f^+\left(\left(1- \frac{d_{\max}-1}{T + d_{\max} -1} \right)\z' + \frac{d_{\max}-1}{T + d_{\max} -1} \bm{0} \right)\\
    &\geq \left(1- \frac{d_{\max}-1}{T + d_{\max} -1} \right)f^+(\z') + \frac{d_{\max}-1}{T + d_{\max} -1} f^+(\bm{0})\\
    &\geq f^+(\z') - \frac{d_{\max}-1}{T + d_{\max} -1} f(\A),
\end{align*}
where the last inequality follows by the facts that $f^+(\bm{0}) = f(\bm{0}) = 0$ and $f^+(\z') \leq f^+(\bm{1}) = f(\A)$, since $f$ is monotone.

Therefore, by exhibiting a feasible solution $\z$ of \eqref{cp:CP} such that 
$$
T f^+(\z) \geq T f^+(\z') - \mathcal{O}(d_{\max} f(\A)) \geq g(\x^*_1, \dots, \x^*_T) - \mathcal{O}(d_{\max} f(\A)) \geq \OPT(T) - \mathcal{O}(d_{\max} f(\A)),
$$
the proof is completed.
\end{proof}

\subsection{Hardness of approximation}
\label{appendix:hardness}
The goal of this section is to show that the $\left(1-\frac{1}{e}\right)$-multiplicative factor in the approximation guarantee of Theorem~\ref{thm:interleavedSubmodular} cannot be improved, unless $\textbf{P} = \textbf{NP}$. Specifically, we prove the following result:

\restateSubmodularHardness* 

In order show the above hardness result, we study for simplicity the average version of \rsm, where the objective is to maximize the average reward over $T$ time steps, namely, $\frac{1}{T}\left(\sum_{t \in [T]} f(\A_t)\right)$, where $\A_t$ is the set of elements used at time $t \in [T]$. Notice that in the average case, the additive term in the approximation guarantee of \ig, as presented in Theorem~\ref{thm:interleavedSubmodular}, vanishes as $T \to \infty$. Let $\OPT$ be the average reward collected by any optimal algorithm for \rsm. 

Our proof relies on a reduction from the Submodular Welfare (SW) problem~\cite{Von08}, in the special case where the players have identical utility functions. The problem can be formally defined as follows:

\begin{definition}[Submodular Welfare with Identical Utilities (SWIU)]
We consider a set of $k$ items and $m$ players, each associated with the same monotone submodular utility function $u:2^{[k]} \rightarrow \mathbb{R}_{\geq 0}$ over the items. The goal is to partition the $k$ items into $m$ subsets $S_1,\dots, S_m$, such that to maximize $\sum_{i \in [m]}u(S_i)$.
\end{definition}

As noted in~\cite{Von08}, the hardness result presented in~\cite{KLMM05} for the SW problem also holds for SWIU, namely, the special case of SW where all the players have the same utility function. Note, also that the \rsm problem is defined in the {\em value oracle} model, as we are only allowed to make queries of the function value for any input set.

\begin{theorem}[\cite{KLMM05}]\label{thm:hardness:swiu}
For any $\epsilon > 0$, there exists no polynomial-time $\left(1 - \frac{1}{e} + \epsilon\right)$-approximation algorithm for the SWIU problem in the value oracle model, unless ${\bf P} = {\bf NP}$.
\end{theorem}

We start from a simple construction for the non-average case of \rsm in order to show how our problem is directly associated with SWIU: Consider an instance of SWIU of $k$ items and $m$ players. Let $u:2^{[k]} \rightarrow \mathbb{R}_{\geq 0}$ be the monotone submodular utility function which is commonly used by all players. Given the above instance, we can construct in polynomial time an instance of \rsm as follows: Let $\A$ be the set of $k$ elements, each corresponding to an item, and let $f:2^{\A} \rightarrow \mathbb{R}_{\geq 0}$ be our function, chosen such that $f \equiv u$. We set the delay of each element $i \in \A$ as well as the time horizon to be equal to the number of players, namely, $ d_i = T  = m$ for each $i \in \A$.

Clearly, in the above construction where the delays are all equal to the time horizon, each element can be chosen at most once by any algorithm for \rsm. Therefore, the above constructed instance of \rsm exactly corresponds to SWIU, given that any solution to latter immediately translates into a solution of \rsm of the same total reward, and the opposite. 

The above construction immediately relates the two problems in the case where the delays can be of the same order as the time horizon. However, it does not rule out the possibility that the \rsm problem might become easier in the special case where $d_{\max} = o(T)$. Indeed, one could argue that for small enough delays, exploiting the possible periodicity of the \rsm solutions might lead to improved approximation guarantees. Notice, further, that the approximation guarantee we provide in Theorem~\ref{thm:interleavedSubmodular} for \IS becomes meaningless in the above scenario, since the additive loss for $d_{\max} = T$ becomes $\mathcal{O}(T\cdot f(\A))$.

In order to overcome the above technical issue and show that the multiplicative factor of $\left(1 - \frac{1}{e}\right)$ in Theorem~\ref{thm:interleavedSubmodular} cannot be improved, we map any instance of SWIU onto an instance of \rsm such that $T \gg d_{\max}$. Given any instance of SWIU, we can construct in polynomial time an instance of \rsm as follows: We define $\A$ to be the set of $k$ items, $f \equiv u$ to be the monotone submodular function and $d_i = m$ $\forall i \in \A$ to be the delay of all elements. In this case, we consider a time horizon $T = m \cdot \lceil\text{poly}(k,m) \rceil$, where by $\text{poly}(k,m)$ we denote some polynomial function in $k$ and $m$. 

We first show that, without loss of generality, we can focus our attention on solutions to the average case of \rsm that exhibit a periodic structure of period $m$.

\begin{lemma}\label{lem:hardness:periodic}
Let $\nu: [T] \rightarrow 2^{\A}$ be any feasible assignment to the above instance of \rsm of average reward $R$. We can construct in polynomial time a feasible assignment $\nu': [T] \rightarrow 2^{\A}$ of average reward at least $R' \geq R$, such that $\nu'(t) = \nu(t + m)$ $\forall t \in \mathbb{N}$, namely, $\nu'$ is a periodic assignment of period $m$.
\end{lemma}
\begin{proof}
Given that the average reward of the assignment $\nu$ is $R$, there must exist a continuous subsequence of rounds of length $m$, that is, $\{t, \dots, t + m -1\}$ for some $t \in [T-m]$, such that
\begin{align*}
    \frac{1}{m}\sum^{t+m-1}_{\tau = t} f(\nu(t)) \geq R.
\end{align*}
In the opposite case, we immediately get a contradiction to the fact that the average reward is at least $R$.

Let $L$ with $|L| = m$ be such a sequence. We now construct the periodic assignment $\nu'$ by repeating the assignment of the subinterval $L$, as follows:
\begin{align*}
\nu'(t) = \nu(L(t \mod m)) \in 2^{\A} ~\forall t \in [T].  
\end{align*}
It is not hard to verify that since $d_i = m$ for each $i \in \A$ and since $L$ is a subsequence of a feasible assignment of length $m$, the assignment $\nu'$ never violates the blocking constraints. Moreover, the average reward of $\nu'$ equals the average reward of the interval $L$ which is at least $R$. Finally, notice that the subsequence $L$ can be found in polynomial time, given the fact that the time horizon $T$ is defined to be polynomial in $k$ and $m$.
\end{proof}

We can now complete the proof of our hardness result. 
\newline

\noindent{\it Proof of Theorem~\ref{thm:submodular:hardness}.}
We prove the result via a reduction from the SWIU problem to the average version of the \rsm. Clearly, the average and non-average version of \rsm share the same approximability status, as the two problems are essentially identical up to a scaling of the objective function. 

Given an instance $I$ of SWIU, we can construct in polynomial time an instance $I'$ of the average version of \rsm, as described above. Let $\OPT_{SWIU}(I)$ and $\OPT_{\rsm}(I')$ be the optimal solution of SWIU and \rsm on the corresponding instance, respectively. 

We first show that when $\OPT_{SWIU}(I) \geq R$ for some reward $R$, then we necessarily have that $\OPT_{\rsm}(I') \geq \frac{R}{m}$. Indeed, let $L:[m] \rightarrow 2^{[k]}$ be an allocation that achieves a reward $R' = \OPT_{SWIU}(I) \geq R$ for the instance $I$ of SWIU. As indicated in proof of Lemma~\ref{lem:hardness:periodic}, we can construct in polynomial time a periodic assignment for the \rsm problem of average reward exactly $\frac{R'}{m}$, which implies that $\OPT_{\rsm}(I') \geq \frac{R'}{m} \geq \frac{R}{m}$.

Now, we would like to show that if $\OPT_{SWIU}(I) \leq \alpha R$ for some reward $R$ and $\alpha \in (0,1)$, then it has to be that $\OPT_{\rsm}(I') \leq \alpha \frac{R}{m}$. We prove the statement via its contrapositive, assuming that $\OPT_{\rsm}(I') > \alpha \frac{R}{m}$ for some reward $R$ and $\alpha \in (0,1)$. Let $\frac{R'}{m}> \alpha \frac{R}{m}$ be the optimal average reward of \rsm. By Lemma~\ref{lem:hardness:periodic}, we can assume w.l.o.g. that the assignment $\OPT_{\rsm}(I')$, that achieves an average reward of $\frac{R'}{m}$, is a periodic assignment of period $m$. However, given that all the delays are equal to $m$ in the instance $I'$ of \rsm, it is easy to see that in any period of $m$ consecutive rounds, each element is played at most once. Moreover, the average reward of each period is exactly $\frac{R'}{m}$. Therefore, any continuous subsequence of length $m$ in the solution of the \rsm naturally induces a solution to the instance $I$ of SWIU of total reward exactly $R'$. This, in turn, implies that $\OPT_{SWIU}(I) \geq R' \geq \alpha R$.

By the above discussion, we have completed the proof of a reduction from SWIU to \rsm. Therefore, any polynomial-time $\left(1 - \frac{1}{e} + \epsilon\right)$-approximation algorithm for \rsm, for some $\epsilon>0$, would imply a $\left(1 - \frac{1}{e} + \epsilon\right)$-approximation algorithm for SWIU. However, by Theorem~\ref{thm:hardness:swiu} this is not possible, unless $\textbf{P} = \textbf{NP}$.
\qed
\newline

We believe that, through a similar reduction as above, we can prove information-theoretic hardness of the \rsm problem by leveraging the results in~\cite{MSV08}. We leave this as future work.
\section{Matroid Blocking Bandits: Omitted Proofs}

\subsection{The full-information setting}
\label{appendix:mbb:fullinformation}
We now provide an alternative analysis for \ig (\IG), which, as opposed to Theorem~\ref{thm:interleavedSubmodular} for the \rsm problem, does not rely on the concave closure of submodular functions. We first note that the correctness of \IG follows directly by Fact~\ref{inter:fact:alwaysavailable}, that is, the set of sampled arms $\G_t$ at each round $t \in [T]$ only contains available (i.e., non-blocked) arms, in combination with the fact that the algorithm always plays an independent set $\A_t \in \I_{\M|\G_t} \subseteq \I$. 

For any set $S \subseteq \A$, let $\z(S) = \sum_{i \in S} z_i$. Consider the following LP, based on the matroid polytope associated with $\M$:
\begin{align}
\textbf{maximize: }& T \cdot \sum_{i \in \A} \mu_{i} z_{i} \tag{\textbf{LP}} \label{lp:LP}\\
\textbf{s.t. }& \z(S) \leq \rk(S), \forall S\subseteq \A \label{flp:rank}\\
\qquad & {\bf 0} \preceq \z \preceq \bm{d}^{-1}. \label{flp:window}
\end{align}

In \eqref{lp:LP}, each variable $z_{i}$ can be thought of as the fraction of rounds where arm $i\in \A$ is played. Intuitively, constraints \eqref{flp:window} of \eqref{lp:LP} indicate the fact that, due to blocking constraints, the fraction of time we can play an arm $i \in \A$ is upper bounded by $\frac{1}{d_i}$, while constraints \eqref{flp:rank} impose the rank restrictions in order to guarantee that the set of arms played at any round $t$ corresponds to an independent set of the matroid $\M$. 

As we show in the following lemma, the formulation \eqref{lp:LP} yields an approximate upper bound on $\OPT(T)$, while the approximation becomes exact as $T$ increases.

\begin{restatable}{lemma}{restateStructuralflp}\label{lem:structural:flp}
Let $\Rew^{LP}(T)$ be the optimal solution to \eqref{lp:LP} and $\OPT(T)$ be the optimal expected reward over $T$ rounds. We have
$
\Rew^{LP}(T) \geq \OPT(T) - \mathcal{O}\left(d_{\max} \rk(\M)\right).
$
\end{restatable}
\begin{proof}
In order to prove the Lemma, we first construct an IP upper bound on the optimal expected reward over $T$ rounds, $\OPT(T)$. Then, we construct \eqref{lp:LP} by averaging over time the 0-1 variables of the IP. For any set $S \subseteq \A$, let $\x_t(S) = \sum_{i \in S} x_{i,t}$. 
\begin{align}
\textbf{maximize:}& \sum_{i \in [T]} \sum_{i \in \A} \mu_{i} x_{i,t} \tag{\textbf{IP}} \label{lp:IP}\\
\textbf{s.t.}& \sum_{t' \in [t,t+d_i-1]} x_{i,t'} \leq 1, \forall i \in \A, \forall t \in [T] \label{lp:window} \\
\qquad & \x_t(S) \leq \rk(S), \forall S\subseteq \A, \forall t \in [T] \label{lp:rank} \\
\qquad & \x_{t} \in \{0,1\}^k, \forall t \in [T]. \notag
\end{align}

In \eqref{lp:IP}, each variable $x_{i,t}$ can be thought of as the 0-1 indicator of playing arm $i\in \A$ at time $t \in [T]$. Intuitively, constraints \eqref{lp:window} of \eqref{lp:IP} indicate the fact that, due to the blocking constraints, each arm $i \in \A$ can be played at most once every $d_i$ steps, while constraints \eqref{lp:rank} impose the rank restrictions due to the matroid $\M$ at any round $t$. Let $\Rew^{IP}(T)$ be the optimal solution to \eqref{lp:IP}.

Fix any (optimal) algorithm and let $\A^{*}_t$ be the set of arms played by the algorithm at round $t$. Notice that the sets $\A^{*}_t$ are deterministic, given that the choices of any full-information algorithm that maximizes the expected cumulative reward are independent of the observed reward realizations. By linearity of expectation, the expected reward collected (over the randomness of the reward realizations) by the optimal algorithm can be expressed as 
\begin{align*}
    \Ex{}{\sum_{t \in [T]} \sum_{i \in \A^{*}_t} X_{i,t}} = \sum_{t \in [T]} \sum_{i \in \A^{*}_t} \Ex{}{X_{i,t}} = \sum_{t \in [T]} \sum_{i \in \A^{*}_t} \mu_i.
\end{align*}
Consider a feasible solution of \eqref{lp:IP} such that for each $i \in \A$ and $t \in [T]$, we set $x_{i,t} = 1$, if $i \in \A^{*}_t$, and $x_{i,t}=0$, otherwise. It is not hard to verify that the objective of \eqref{lp:IP} for this assignment coincides with the expected reward collected by the above optimal algorithm. Moreover, constraints \eqref{lp:window} are satisfied, since for any arm $i \in [T]$ and any window of $d_i$ consecutive time steps, the algorithm can play the arm at most once. Finally, constraints \eqref{lp:rank} are satisfied, since for any time $t$, the set of arms played, $\A^{*}_t$, is an independent set of the matroid $\M$, thus satisfying all the rank constraints. Therefore, by exhibiting a feasible solution of \eqref{lp:IP} that has the same objective value as the expected reward of any optimal algorithm, we conclude that $\Rew^{IP}(T) \geq \OPT(T)$.

Consider any optimal solution $\x^*$ of \eqref{lp:IP} for a time horizon $T$. By constraints \eqref{lp:window}, for any $t \in [T]$ and $i \in \A$, we have $\sum_{t' \in [t,t+d_i-1]} x^*_{i,t'} \leq 1$. By working along the lines of the proof of Lemma~\ref{lem:structural:CP} and averaging constraints \eqref{lp:window} over all $t \in [T]$, we get
\begin{align}
    \frac{1}{T}\sum_{t \in [T]} x^*_{i,t} \leq \frac{1}{d_i} \left(1 + \frac{d_i-1}{T}\right), \forall i \in \A \label{eq:flp:window}. 
\end{align}
Similarly, for any set $S \subseteq \A$, by averaging the inequalities of \eqref{lp:rank} over all rounds $t \in [T]$, we get
\begin{align}
    \frac{1}{T}\sum_{t \in [T]} \x^*_t(S) \leq \rk(S), \forall S \subseteq \A. \label{eq:flp:rank}
\end{align}
Now, consider an assignment of \eqref{lp:LP} such that
\begin{align*}
z_i = \left(1 + \frac{d_{\max}-1}{T}\right)^{-1} \frac{1}{T} \sum_{t \in [T]} x^*_{i,t},~~\forall i \in \A.    
\end{align*}

It is not hard to see that by inequality \eqref{eq:flp:window}, we have $z_{i} \leq \frac{1}{d_i}$ for any $i \in \A$. Moreover, given that $\left(1 + \frac{d_{\max}-1}{T}\right)^{-1} \leq 1$, for any set $S \subseteq \A$, we have that $\sum_{i \in S} z_i \leq \frac{1}{T}\sum_{t \in [T]} \x^*_t(S) \leq \rk(S)$. Therefore, the assignment $\z \in \mathbb{R}_{+}$ with $(\z)_i = z_i$ satisfies constraints \eqref{flp:rank} and \eqref{flp:window} of \eqref{lp:LP}. Considering the objective value of \eqref{lp:LP} for the assignment $\z$, we have that
\begin{align*}
    T \sum_{i \in \A} \mu_i z_i &= T  \sum_{i \in \A} \mu_i \left(1 + \frac{d_{\max}-1}{T}\right)^{-1} \frac{1}{T} \sum_{t \in [T]} x^*_{i,t} \\
    &\geq \left(1 + \frac{d_{\max}-1}{T}\right)^{-1} \sum_{t \in [T]} \sum_{i \in \A} \mu_i x^*_{i,t}\\
    &\geq \left(1 - \frac{d_{\max}-1}{d_{\max}-1 + T}\right) \sum_{t \in [T]} \sum_{i \in \A} \mu_i x^*_{i,t},
\end{align*}
where the last inequality follows by the fact that $\frac{1}{1+\beta} = 1 - \frac{\beta}{1+ \beta}$ for any $\beta \in \mathbb{R}$.
By exhibiting a feasible solution of \eqref{lp:LP} of value greater than $\left(1 - \frac{d_{\max}-1}{d_{\max}-1 + T}\right) \Rew^{IP}(T)$, the lemma follows by the fact that $\Rew^{IP}(T) \geq \OPT(T)$ and that $\OPT(T) \leq T \cdot \rk(\M)$, since the rewards of all arms lie in $[0,1]$.
\end{proof}

We are now ready to complete the proof of the following result.

\restateinterleavedGreedy*

\begin{proof}
Before we proceed with the proof, we first emphasize that the algorithm \IG is not aware of the reward realizations of each round before it plays a subset of arms. Therefore, since the objective it to maximize the cumulative expected reward, we can assume that the reward of each arm $i \in \A$ is deterministic and equal to $\mu_i$.

Let $\z^*$ be an optimal solution to \eqref{lp:LP}. Given the fact that the feasible set of \eqref{lp:LP} is essentially the intersection of the matroid polytope $\mathcal{P}(\M)$ and the (downward-closed) blocking constraints $\z^* \leq \bm{d}^{-1}$, it holds that $\z^* \in \mathcal{P}(\M)$. Therefore, the point $\z^*$ can be expressed as a convex combination of characteristic vectors of $k$ independent sets of $\M$, denoted by $T_1, \dots, T_k $, where $T_j \in \I, \forall j \in [k]$. By Lemma \ref{lem:characteristic}, this in turn induces a probability distribution, $\I(\z^*)$, over $T_1, \dots, T_k$, such that the marginal probability of each element $i \in \A$ being in the sampled set is exactly $z^*_i$.

Conditioned on the random offsets $\{r_i\}_{i \in \A}$, the sequence of sampled sets $\{\G_t\}_{t \in [T]}$ is deterministic and independent of the observed rewards. Let $f_{\M,\mu}(\G_t)$ be the weighted rank function over the subset $\G_t$, that is, the expected reward of a maximum independent set of $\M$ contained in $\G_t$. By denoting as $\G_t \sim {\bf p}$ the random set of elements, where each element $i \in \A$ participates with probability equal to $({\bf p})_i = p_i$, we have that $\G_t \sim {\bf d}^{-1}$ for each $t \in [T]$. The expected reward of \IG can be expressed as
\begin{align*}
    \Rew^{IG}(T) = \sum_{t \in [T]} \E{\mu(\A_t)} = \sum_{t \in [T]} \Ex{\G_t \sim {\bf d}^{-1}}{f_{\M,\mu}(\G_t)} \geq  \sum_{t \in [T]} \Ex{\G_t \sim \z^*}{f_{\M,\mu}(\G_t)},
\end{align*}
where the last inequality follows by Lemma \ref{lem:weightedrank}, namely, the fact that the weighted rank function $f_{\M,\mu}(\G_t)$ is a monotone (non-decreasing) and by the fact that $\z^* \preceq \bm{d}^{-1}$.

Let $F_{\M,\mu}(\z)$ and $f_{\M,\mu}^+(\z)$ be the multi-linear extension and the concave closure of function $f_{\M,\mu}$, respectively. By the correlation gap inequality for submodular functions (see Lemma \ref{lem:correlationgap}), for each vector $\z$, we have that $F_{\M,\mu}(\z) \geq \left(1 - \frac{1}{e}\right)f_{\M,\mu}^+(\z)$. Moreover, by definition of the concave closure, it has to be that $f_{\M,\mu}^+(\z) \geq \Ex{I \sim \I(\z)}{f_{\M, \mu}(I)}$, since $f_{\M,\mu}^+(\z)$ is the maximum valued distribution over independent sets, such that the marginal contribution of each element $i \in \A$ is equal to $z_i$, i.e., $\Prob{I \sim \I(\z)}{i \in I} = z_i$. By combining the above facts, we have that
\begin{align*}
    \sum_{t \in [T]} \Ex{\G_t \sim \z^*}{f_{\M,\mu}(\G_t)} = T \cdot F_{\M,\mu}(\z^*) \geq \left(1 - \frac{1}{e}\right) T \cdot f^+_{\M,\mu}(\z^*) \geq \left(1 - \frac{1}{e}\right) T \cdot \Ex{I \sim \I(\z^*)}{f_{\M, \mu}(I)}.
\end{align*}

Using the fact that the greedy algorithm collects every element in $I$ for any independent set $I \in \I$, we have that $\Ex{I \sim \I(\z^*)}{f_{\M, \mu}(I)} = \Ex{I \sim \I(\z^*)}{\mu(I)}$. Finally, since the marginal probability of each element $i \in \A$ being in $I \sim \I(\z^*)$ equals $z^*_i$, we have
\begin{align*}
T \cdot \Ex{I \sim \I(\z^*)}{f_{\M, \mu}(I)} = T \cdot \Ex{I \sim \I(\z^*)}{\mu(I)} = T \cdot \sum_{i \in \A} \mu_i z^*_i = \Rew^{LP}(T).
\end{align*}
By combining the above relations with Lemma \ref{lem:structural:flp}, we get that
\begin{align*}
    \Rew^{IG}(T) \geq \left(1 - \frac{1}{e}\right) \Rew^{LP}(T) \geq \left(1 - \frac{1}{e}\right) \OPT(T) - \mathcal{O}(d_{\max} \rk(\M)),
\end{align*}
thus, the proof is completed.
\end{proof}

\subsection{The bandit setting and regret analysis}

\restateRegretCoupling*
\begin{proof}
Let $\{\G_t(\off)\}_{t \in [T]}$ be the sequence of sampled arms over $T$ rounds as a function of the sampled offsets $\off \in [0,1]^k$. Moreover, let $X_t(S)$ be the realized rewards of a subset $S \subseteq \A$ of arms at round $t \in [T]$. We denote by $\A^{\pi}_t$ the arms played at round $t\in[T]$ and by $H^{\pi}_t = \{\A^{\pi}_1, X_1(\A^{\pi}_1), \dots, \A^{\pi}_t, X_t(\A^{\pi}_t)\}$ the {\em history} of arm playing and observed realizations up to (and including) time $t$ by algorithm $\pi \in \{\IG, \UCB\}$. Recall that we denote by $\mathcal{R}$ the randomness due to the reward realizations of the arms.

Notice that in the case of \UCB and for fixed offsets, the player's actions only depend on the previous realized rewards of the arms. Thus, for any fixed offset vector $\off^{\UCB}$, we have
\begin{align*}
&\Ex{\mathcal{R}}{\sum_{i \in \A} X_{i,t} \event{i \in \arg\max_{S \subseteq \G_t(\off^{\UCB}), S \in \I}\{\bar{\mu}_t(S)\}} }\\
    &= \Ex{\mathcal{R}}{\sum_{i \in \A} \Ex{\mathcal{R}}{ X_{i,t} \event{i \in \arg\max_{S \subseteq \G_t(\off^{\UCB}), S \in \I}\{\bar{\mu}_t(S)\}}~|~H^{\UCB}_{t-1}}}\\
    &= \Ex{\mathcal{R}}{\sum_{i \in \A} \Ex{\mathcal{R}}{ X_{i,t}~|~H^{\UCB}_{t-1}} \event{i \in \arg\max_{S \subseteq \G_t(\off^{\UCB}), S \in \I}\{\bar{\mu}_t(S)\}}} \\
    &= \Ex{\mathcal{R}}{\sum_{i \in \A} \mu_i \event{i \in \arg\max_{S \subseteq \G_t(\off^{\UCB}), S \in \I}\{\bar{\mu}_t(S)\}}}\\
    &=\Ex{\mathcal{R}}{ \mu\left( \arg\max_{S \subseteq \G_t(\off^{\UCB}), S \in \I}\{\bar{\mu}_t(S)\}\right)}.
\end{align*}

Similarly, notice that the algorithm \IG is oblivious to the realized rewards. Therefore, for any fixed offset vector $\off^{\IG}$ and at any time $t \in [T]$, we get
\begin{align*}
    \Ex{\mathcal{R}}{\sum_{i \in \A} X_{i,t} \event{i \in \arg\max_{S \subseteq \G_t(\off^{\IG}), S \in \I}\{{\mu}(S)\}} }
    &= 
    \Ex{\mathcal{R}}{\sum_{i \in \A} \mu_i \event{i \in \arg\max_{S \subseteq \G_t(\off^{\IG}), S \in \I}\{\mu(S)\}}}\\
    &=\Ex{\mathcal{R}}{ \max_{S \subseteq \G_t(\off^{\IG}), S \in \I}\{\mu(S)\}}.
\end{align*}
The lemma follows by observing that the offsets $\off^{\IG}$ and $\off^{\UCB}$ of the two algorithms follow exactly the same distribution. Therefore, we have

\begin{align*}
&\Rew^{\IG}(T) - \Rew^{\UCB}(T) \\
= & \Ex{\off^{\IG} \sim [0,1]^k, \mathcal{R}}{\sum_{t \in [T]}\max_{S \subseteq \G_t(\off^{\IG}), S \in \I}\{\mu(S)\}} - \Ex{\off^{\UCB} \sim [0,1]^k, \mathcal{R}}{\sum_{t \in [T]}\mu\left( \arg\max_{S \subseteq \G_t(\off^{\UCB}), S \in \I}\{\bar{\mu}_t(S)\}\right)} \\
= & \Ex{\off \sim [0,1]^k, \mathcal{R}}{\sum_{t \in [T]}\left(\max_{S \subseteq \G_t(\off), S \in \I}\{\mu(S)\} - \mu\left( \arg\max_{S \subseteq \G_t(\off), S \in \I}\{\bar{\mu}_t(S)\}\right)\right)}.
\end{align*}
\end{proof}

\restateRegretSub*

\begin{proof}
Recall that under the event $\mathcal{E}_{\off}$, both algorithms \IG and \UCB use the same offset vector $\off$ and, thus, they operate on same sequence of sampled arms over time. Let $\G_t = \G_t(\off)$ be the common set of sampled arms and let $\A^{\IG}_t$ and $\A^{\UCB}_t$ be the maximal independent sets computed by \IG and \UCB, respectively, at any round $t \in [T]$. Notice that for any $t \in [T]$ both $\A^{\IG}_t$ and $\A^{\UCB}_t$ are bases of the restricted matroid $\M|\G_t$ and, thus, correspond to independent sets of $\I$ of equal cardinality. Let $\sigma_t$ be the bijection between $\A^{\IG}_t$ and $\A^{\UCB}_t$ described by Theorem~\ref{cor:basesexchange}. For any $t \in [T]$, we have that
$$
\Delta_{\G_t}(\bar\mu) =  \mu(\A^{\IG}_{t}) - \mu(\A^{\UCB}_{t}) = \sum_{i \in \A_t^{\IG}} \mu_i - \sum_{j \in \A_t^{\UCB}} \mu_j = \sum_{i \in \A^{\IG}_t} \left(\mu_i - \mu_{\sigma^{-1}_t(i)}\right) = \sum_{i \in \A^{\IG}_t} \Delta_{i , \sigma^{-1}_{t}(i)}.
$$
\end{proof}

\restateRegretTechnical*

\begin{proof}
We first focus on proving inequality \eqref{inq:reg:under}, that is, the part of the regret attributed to an arm $j >1$ when not enough samples have been collected. Notice that the algorithm $\UCB$ never accumulates regret when it plays the arm $j=1$ of highest mean reward. Recall that for any fixed $j \in \A$, we have $\Delta_{1,j} > \Delta_{2,j} > \dots > \Delta_{j,j} = 0$, since we assume w.l.o.g. that the arms have distinct mean rewards. By construction of our algorithm, if the number of samples from arm $j \in \A$ is increased at some round $t$, it is because there exists exactly one arm $i \in \A$ with $\Delta_{i,j} > 0$, such that $\sigma_t(j) = i$. The above is implied by Theorem~\ref{cor:basesexchange}, given the fact that each bijection $\sigma_t$ for all $t \in [T]$ maps each arm played by \UCB in $\A^{\UCB}_t$ to a single arm played by \IG in $\A_t^{\IG}$. On the other hand, as the number of obtained samples $T_j(t)$ from arm $j \in \A$ by time $t\in [T]$ increases, the maximum suboptimality gap $\Delta_{i,j}$ that can be charged in the under-sampled part of the regret is that of the maximum reward $i \in \A$ that satisfies $T_j(t) \leq \ell_{i,j}$. By the above analysis, for any $j>1$, we get that 
\begin{align}
\sum_{t \in [T]} \sum^{j-1}_{i = 1} \Delta_{i , j} \event{\sigma_t(j)=i, T_j(t) \leq \ell_{i,j}} \notag
&\leq \sum^{j-1}_{i=1} \left(\Delta_{i,j} -  \Delta_{i+1,j}\right)\ell_{i,j} \notag\\
&\leq \sum^{j-1}_{i=1} \left(\Delta_{i,j} -  \Delta_{i+1,j}\right)\frac{8 \ln(T)}{\Delta^2_{i,j}} \label{eq:regret:technical:1},
\end{align}
where the last inequality follows by definition of $\ell_{i,j}$.

The rest of the claim follows by simple algebra. Indeed,
\begin{align*}
    \eqref{eq:regret:technical:1}&\leq \left(\sum^{j-1}_{i=1}\frac{\Delta_{i,j} - \Delta_{i+1,j}}{\Delta^2_{i,j}}\right)8 \ln(T) \notag\\
    &\leq \left(\frac{1}{\Delta_{j-1,j}} + \sum^{j-2}_{i=1}\frac{\Delta_{i,j} - \Delta_{i+1,j}}{\Delta^2_{i,j}}\right)8 \ln(T) \notag\\
    &\leq \left(\frac{1}{\Delta_{j-1,j}} + \sum^{j-2}_{i=1}\frac{\Delta_{i,j} - \Delta_{i+1,j}}{\Delta_{i,j} \Delta_{i+1,j}}\right)8 \ln(T) \notag\\
    &= \left(\frac{1}{\Delta_{j-1,j}} + \sum^{j-2}_{i=1}\left(\frac{1}{\Delta_{i+1,j}} - \frac{1}{\Delta_{i,j}}\right)\right)8 \ln(T) \notag\\
    &= \left(\frac{2}{\Delta_{j-1,j}} - \frac{1}{\Delta_{1,j}}\right)8 \ln(T) \notag\\
    &\leq \frac{16}{\Delta_{j-1,j}} \ln(T) \notag.
\end{align*}

We now focus on proving inequality~\eqref{inq:reg:sufficiently}, that is, the regret accumulated after a sufficient number of samples has been collected from an arm $j > 1$. Notice, that given the event $\mathcal{E}_{\off}$, the expectation in the LHS of inequality~\eqref{inq:reg:sufficiently} is taken only over the randomness of the realized rewards that are observed by \UCB.

For proving the upper bound, we fix any arm $j > 1$ and focus on each arm $i \in \A$ such that $i < j$ and, thus, $\Delta_{i,j}>0$. Let us fix any such arm $i \in \A$. For any $t \in [T]$, the event $\{\sigma_t(j) = i\}$ implies that $\{\mu_{i} > \mu_j, \bar{\mu}_{i,t} \leq \bar{\mu}_{j,t}\}$, namely, the order of the UCB-indices at time $t \in [T]$ of $i$ and $j$ is inconsistent with the order of their true mean rewards. In the opposite case, the algorithm \UCB would have chosen the set $\A^{\UCB}_t - j + i$, which, as suggested by Theorem~\ref{cor:basesexchange}, is an independent set of $\M$. Therefore, for any arm $i < j$, we have
\begin{align}
    \{\sigma_t(j)=i, T_j(t) > \ell_{i,j}\} \subseteq \{\bar{\mu}_{i,t} \leq \bar{\mu}_{j,t},\mu_i > \mu_j, T_j(t) > \ell_{i,j}\}. \label{eq:lem:ucb:0}
\end{align}
Note that the inclusion in the above expression is because the inconsistency in the order of UCB-indices does not necessarily imply that $\sigma_t(j)=i$ (i.e., that \UCB actually exchanges $j$ for $i$ at time $t \in [T]$).

By definition of the UCB-indices, the event $\bar{\mu}_{i,t} \leq \bar{\mu}_{j,t}$ at time $t \in [T]$ implies that 
\begin{align}
    \hat{\mu}_{i,T_{i}(t)} + \sqrt{\frac{2 \ln{(t)}}{T_{i}(t)}} \leq \hat{\mu}_{j,T_{j}(t)} + \sqrt{\frac{2 \ln{(t)}}{T_{j}(t)}}. \label{eq:lem:ucb:1}
\end{align}

We fix $s_i = T_i(t)$ and $s_j = T_j(t) > \ell_{i,j}$ to be the number of samples obtained from arm $i$ and $j$, respectively, by time $t \in [T]$. Notice that in order for \eqref{eq:lem:ucb:1} to hold, at least one of the following events must be true:
\begin{align*}
     \textbf{(i) }\bigg\{\hat{\mu}_{i,s_{i}} + \sqrt{\frac{2 \ln{(t)}}{s_i}} \leq \mu_i \bigg\}, \textbf{   (ii)  }\bigg\{\hat{\mu}_{j,s_{j}} \geq \mu_j + \sqrt{\frac{2 \ln{(t)}}{s_{j}}}\bigg\},\textbf{   (iii)  } \bigg\{\mu_i < \mu_j + 2 \sqrt{\frac{2 \ln{(t)}}{s_j}}\bigg\}.
\end{align*}
Indeed, it can be easily verified that the simultaneous negation of the above three events contradicts \eqref{eq:lem:ucb:1} for any fixed number of samples $s_i,s_j$. 

By our choice of $\ell_{i,j} = \bigg\lfloor \frac{8 \ln(T)}{\Delta^2_{i,j}}\bigg\rfloor$ and the fact that $s_j \geq \ell_{i,j} + 1 \geq \frac{8 \ln(T)}{\Delta^2_{i,j}}$, we can see that event $\textbf{(iii)}$ cannot be true, since in that case, we have
$$
 \mu_j + 2 \sqrt{\frac{2 \ln{(t)}}{s_j}}  \leq \mu_j + 2 \sqrt{\frac{2 \Delta^2_{i,j}\ln{(t)}}{8 \ln(T)}} \leq \mu_j + \Delta_{i,j} = \mu_i.
$$
Moreover, by Hoeffding's inequality, for the probabilities of the events $\textbf{(i)}$ and $\textbf{(ii)}$, we have that
$$
\Pro{\hat{\mu}_{i,s_{i}} + \sqrt{\frac{2 \ln{(t)}}{s_i}} \leq \mu_i} \leq e^{-4 \ln(t)} = t^{-4}\text{   and   }\Pro{\hat{\mu}_{j,s_{j}} \geq \mu_j + \sqrt{\frac{2 \ln{(t)}}{s_{j}}}} \leq e^{-4 \ln(t)} = t^{-4},
$$
where the probability is taken over the randomness of the reward realizations.

Therefore, for any numbers of samples $s_i = T_i(t)$ and $s_j = T_j(t) > \ell_{i,j}$, we have
\begin{align}
    \Pro{\bar{\mu}_{i,t} \leq \bar{\mu}_{j,t},\mu_i > \mu_j, T_j(t) = s_j, T_i(t)= s_i} 
    &\leq \Pro{\hat{\mu}_{i,s_{i}} + \sqrt{\frac{2 \ln{(t)}}{s_i}} \leq \mu_i} + \Pro{\hat{\mu}_{j,s_{j}} \geq \mu_j + \sqrt{\frac{2 \ln{(t)}}{s_{j}}}} \notag\\
    &\leq 2\cdot t^{-4}  \label{eq:ucb:union}.
\end{align}
Finally, by union bound over the possible number of samples, $s_i$ and $s_j$, and using the aforementioned results, for any $j > 1$ and time $t \in [T]$, we have
\begin{align}
    &\Ex{\mathcal{R}}{\sum_{t \in [T]} \sum^{j-1}_{i = 1} \Delta_{i , j} \event{\sigma_t(j)=i, T_j(t) > \ell_{i,j}}} \notag\\
    &= \Ex{\mathcal{R}}{\sum_{t \in [T]} \sum^{j-1}_{i = 1} \sum^{t-1}_{s_i = 0} \sum^{t-1}_{s_j = \ell_{i,j}+1} \Delta_{i , j} \event{\sigma_t(j)=i, T_j(t) = s_j, T_i(t)= s_i}}\label{lem:ucb:f2}\\
    &\leq \Ex{\mathcal{R}}{\sum_{t \in [T]} \sum^{j-1}_{i = 1} \sum^{t-1}_{s_i = 0} \sum^{t-1}_{s_j = \ell_{i,j}+1} \Delta_{i , j} \event{\bar{\mu}_{i,t} \leq \bar{\mu}_{j,t},\mu_i > \mu_j, T_j(t) = s_j, T_i(t)= s_i}}\label{lem:ucb:f3}\\
    &= \sum_{t \in [T]} \sum^{j-1}_{i = 1} \sum^{t-1}_{s_i = 0} \sum^{t-1}_{s_j = \ell_{i,j}+1} \Delta_{i , j} \Pro{\bar{\mu}_{i,t} \leq \bar{\mu}_{j,t},\mu_i > \mu_j, T_j(t) = s_j, T_i(t)= s_i}\notag\\
    &\leq  \sum_{t \in [T]} \sum^{j-1}_{i = 1} \Delta_{i , j} 2 t(t-1) t^{-4}  \label{lem:ucb:f4},
\end{align}
where in \eqref{lem:ucb:f2} we consider any possible number of samples by time $t$ for each arm. Moreover, inequality \eqref{lem:ucb:f3} follows by \eqref{eq:lem:ucb:0} and \eqref{lem:ucb:f4} follows by \eqref{eq:ucb:union}.
The proof of inequality \eqref{inq:reg:sufficiently} follows by the fact that 
$$\sum_{t \in [T]} t(t-1)t^{-4} \leq \sum_{t \in [T]}t^{-2} \leq \sum^{+\infty}_{t =1}t^{-2} = \frac{\pi^2}{6}.$$
\end{proof}

\subsection{Proof of Theorem \ref{thm:approxregret}}

\restateApproxRegret*

\begin{proof}
By inequality \eqref{eq:regret:twoalgos}, Lemma~\ref{lem:regret:coupling} and Definition \ref{def:gaps}, we can upper bound the $\alpha$-regret, for $\alpha = 1 - \frac{1}{e}$, as
\begin{align}
\alpha \OPT(T) - \Rew^{\UCB}(T) \leq \Ex{\off \sim [0,1]^k, \mathcal{R}}{\sum_{t \in [T]}\Delta_{\G_t(\off)}(\bar{\mu}_t)} + \mathcal{O}(d_{\max}\cdot \rk(\M)), \label{eq:reg:final:1}
\end{align}
where the expectation is taken over the randomness of the offset vector $\off$ and the reward realizations.

Under the event $\mathcal{E}_{\off}$, that is, where both \IG and \UCB use the same offsets $\off$, let $\{\sigma_t\}_{t \in [T]}$ be the sequence of bijections between $\A_t^{\UCB}$ and $\A_t^{\IG}$ over all rounds $t \in [T]$, as described in Theorem~\ref{cor:basesexchange}. Using Lemma~\ref{lem:regret:suboptdecomp}, we have that
\begin{align}
    \Ex{\off \sim [0,1]^k, \mathcal{R}}{\sum_{t \in [T]}\Delta_{\G_t(\off)}(\mu_t)} 
    &= \Ex{\off \sim [0,1]^k, \mathcal{R}}{\sum_{t \in [T]}\sum_{i \in \A^{\IG}_t} \Delta_{i , \sigma^{-1}_{t}(i)}}\notag \\
    &= \Ex{\off \sim [0,1]^k, \mathcal{R}}{\sum_{t \in [T]}\sum_{i \in \A^{\IG}_t}
    \sum_{j \in \A} \Delta_{i,j} \event{\sigma_t(j) = i}}\notag\\
    &\leq \Ex{\off \sim [0,1]^k, \mathcal{R}}{\sum_{t \in [T]}\sum_{j \in \A}\sum_{i < j}
    \Delta_{i,j} \event{\sigma_t(j) = i}}, \label{eq:reg:final:2}
\end{align}
where in the last inequality we restrict ourselves to arms $i < j$, where $\Delta_{i,j}>0$.

Now using the results of Lemma~\ref{lem:regret:technical}, we can further upper bound \eqref{eq:reg:final:2} as 

\begin{align}
&\Ex{\off \sim [0,1]^k, \mathcal{R}}{\sum_{t \in [T]}\sum_{j \in \A}\sum_{i < j} \Delta_{i,j} \event{\sigma_t(j) = i}} \notag \\
&= \Ex{\off \sim [0,1]^k, \mathcal{R}}{\sum_{t \in [T]}\sum_{j \in \A}\sum_{i < j} \Delta_{i,j} \event{\sigma_t(j) = i, T_j(t) \leq \ell_{i,j}}} \notag\\
&\qquad \qquad+ \Ex{\off \sim [0,1]^k}{ \Ex{\mathcal{R}}{\sum_{t \in [T]}\sum_{j \in \A}\sum_{i < j} \Delta_{i,j} \event{\sigma_t(j) = i, T_j(t) > \ell_{i,j}}}} \notag \\
&\leq \sum_{j > 1} \frac{16}{\Delta_{j-1,j}}\ln(T) + \frac{\pi^2}{3} \sum_{j > 1}\sum_{i = 1}^{j-1} \Delta_{i,j}. \label{eq:reg:final:3}
\end{align}

By combining inequalities \eqref{eq:reg:final:1}, \eqref{eq:reg:final:2} and \eqref{eq:reg:final:3}, we can upper bound the regret as a function of the gaps as follows:
\begin{align*}
    \alpha \OPT(T) - \Rew^{\UCB}(T) \leq \sum_{j > 1} \frac{16}{\Delta_{j-1,j}}\ln(T) + \frac{\pi^2}{3} \sum_{j > 1}\sum_{i = 1}^{j-1} \Delta_{i,j} + \mathcal{O}(d_{\max}\cdot \rk(\M))~~ \text{ (gap-dependent regret)}.
\end{align*}

In order to conclude the proof of the theorem, we would like to construct a regret bound that is independent of the gaps. The standard method is to partition the suboptimality gaps into ``small'' and ``large'' and, then, separately study their contribution to the regret. Specifically, for each $j \in \A$ and fixed $\epsilon > 0$, we define:
\begin{align*}
    S_j = \{i < j~|~\Delta_{i,j} \leq \epsilon\}\text{ and }L_j = \{i < j~|~\Delta_{i,j} > \epsilon\}.
\end{align*}
Starting again from \eqref{eq:reg:final:2} and noticing that the total regret due to small gaps can be at most $\epsilon\cdot T$ per arm, we have
\begin{align}
&\Ex{\off \sim [0,1]^k, \mathcal{R}}{\sum_{t \in [T]}\sum_{j \in \A}\sum_{i < j} \Delta_{i,j} \event{\sigma_t(j) = i}} \notag \\
&= \Ex{\off \sim [0,1]^k, \mathcal{R}}{\sum_{t \in [T]}\sum_{j \in \A}\sum_{i \in S_j} \Delta_{i,j} \event{\sigma_t(j) = i}} + \Ex{\off \sim [0,1]^k, \mathcal{R}}{\sum_{t \in [T]}\sum_{j \in \A}\sum_{i \in L_j} \Delta_{i,j} \event{\sigma_t(j) = i}} \notag\\
&\leq 
\epsilon k  T + \Ex{\off \sim [0,1]^k, \mathcal{R}}{\sum_{t \in [T]}\sum_{j \in \A}\sum_{i \in L_j} \Delta_{i,j} \event{\sigma_t(j) = i}}. \label{eq:reg:final:4}
\end{align}

We now focus only on the regret due to the large gaps, namely, the pairs $i,j$ such that $j \in \A$ and $i \in L_j$, which implies that $\Delta_{i,j} > \epsilon$. By exactly the same analysis as in the gap-dependent case, we can reach inequality \eqref{eq:reg:final:3}, in the restricted case where the summations only include pairs of arms such that $\Delta_{i,j} > \epsilon$ (notice that we can apply Lemma~\ref{lem:regret:technical} considering only the set $L_j$ of arms for each $j>1$). In addition, using the fact that $\Delta_{i,j} \leq 1$ for any $i,j \in \A$, we have 
\begin{align}
    \Ex{\off \sim [0,1]^k, \mathcal{R}}{\sum_{t \in [T]}\sum_{j \in \A}\sum_{i \in L_j} \Delta_{i,j} \event{\sigma_t(j) = i}} \leq 
    \sum_{j >1 } \frac{16}{\epsilon}\ln(T) + \frac{\pi^2}{6} k (k-1). \label{eq:reg:final:5}
\end{align}
By combining inequalities \eqref{eq:reg:final:4} and \eqref{eq:reg:final:5} with \eqref{eq:reg:final:1} and \eqref{eq:reg:final:2}, we have
\begin{align*}
    \alpha \OPT(T) - \Rew^{\UCB}(T) \leq \epsilon k T + \frac{16 k }{\epsilon}\ln(T) + \frac{\pi^2}{6} k(k-1) + \mathcal{O}(d_{\max}\cdot \rk(\M)).
\end{align*}
Finally, by setting $\epsilon = 4 \sqrt{\frac{\ln(T)}{T}}$, we get that 
\begin{align*}
   \alpha \OPT(T) - \Rew^{\UCB}(T) \leq 8 k \sqrt{T \ln(T)} + \frac{\pi^2}{6} k(k-1) + \mathcal{O}(d_{\max}\cdot \rk(\M))\quad \text{ (gap-independent regret)}.
\end{align*}
Therefore, we can conclude that the expected reward collected by \UCB in $T$ rounds is at least
\begin{align*}
    \left(1-\frac{1}{e}\right)\OPT(T) - \mathcal{O}\left(k \sqrt{T \ln(T)} + k^2 + d_{\max}\cdot \rk(\M) \right).
\end{align*}
\end{proof}
\section{Additional Results}

\subsection{Tight example for the naive greedy algorithm} \label{appendix:tightexample}

\begin{restatable}{lemma}{restateTightexample}\label{lem:tightexample}
For any $d \geq 2$, there exists an instance of the full-information variant of the \mbb problem (where the mean rewards are known a priori) such that the greedy strategy that plays a maximum mean reward independent set among the available arms collects a $\left(\frac{1}{2} + \frac{1}{2d}\right)$-fraction of the optimal expected reward.
\end{restatable}

\begin{proof}
We consider an infinite time horizon and a graphic matroid based on the graph $G_d = (V_d, E_d)$, which is recursively defined as follows: Let $G_1 = (V_1, E_1)$ with $V_1 = \{u,v\}$, $E_1 = \{\{u,v\}\}$ and assume that the arm associated with edge $\{u,v\}$ has delay $1$ and mean reward $1-\epsilon$, for some $\epsilon > 0$. For the graph $G_d = (V_d, E_d)$, we have $V_d = V_{d-1} \cup \{u_d\}$ and $E_d = E_{d-1} \cup \{\{u,u_d\}, \forall u \in V_{d-1}\}$ (namely, $G_d$ is essentially the result of the join operation between $G_{d-1}$ and a single vertex graph). The arms that are associated with the edges of $E_{d} \setminus E_{d-1}$ all have delay equal to $d$ and mean reward equal to $1- \frac{\epsilon}{d}$. The above recursive construction is illustrated in Figure \ref{fig:graphicmatroid}.

\begin{figure}[h]
\centering
\tikzstyle{ipe stylesheet} = [
  ipe import,
  even odd rule,
  line join=round,
  line cap=butt,
  ipe pen normal/.style={line width=0.4},
  ipe pen heavier/.style={line width=0.8},
  ipe pen fat/.style={line width=1.2},
  ipe pen ultrafat/.style={line width=2},
  ipe pen normal,
  ipe mark normal/.style={ipe mark scale=3},
  ipe mark large/.style={ipe mark scale=5},
  ipe mark small/.style={ipe mark scale=2},
  ipe mark tiny/.style={ipe mark scale=1.1},
  ipe mark normal,
  /pgf/arrow keys/.cd,
  ipe arrow normal/.style={scale=7},
  ipe arrow large/.style={scale=10},
  ipe arrow small/.style={scale=5},
  ipe arrow tiny/.style={scale=3},
  ipe arrow normal,
  /tikz/.cd,
  ipe arrows, % update arrows
  <->/.tip = ipe normal,
  ipe dash normal/.style={dash pattern=},
  ipe dash dashed/.style={dash pattern=on 4bp off 4bp},
  ipe dash dotted/.style={dash pattern=on 1bp off 3bp},
  ipe dash dash dotted/.style={dash pattern=on 4bp off 2bp on 1bp off 2bp},
  ipe dash dash dot dotted/.style={dash pattern=on 4bp off 2bp on 1bp off 2bp on 1bp off 2bp},
  ipe dash normal,
  ipe node/.append style={font=\normalsize},
  ipe stretch normal/.style={ipe node stretch=1},
  ipe stretch normal,
  ipe opacity 10/.style={opacity=0.1},
  ipe opacity 30/.style={opacity=0.3},
  ipe opacity 50/.style={opacity=0.5},
  ipe opacity 75/.style={opacity=0.75},
  ipe opacity opaque/.style={opacity=1},
  ipe opacity opaque,
]
\definecolor{red}{rgb}{1,0,0}
\definecolor{green}{rgb}{0,1,0}
\definecolor{blue}{rgb}{0,0,1}
\definecolor{yellow}{rgb}{1,1,0}
\definecolor{orange}{rgb}{1,0.647,0}
\definecolor{gold}{rgb}{1,0.843,0}
\definecolor{purple}{rgb}{0.627,0.125,0.941}
\definecolor{gray}{rgb}{0.745,0.745,0.745}
\definecolor{brown}{rgb}{0.647,0.165,0.165}
\definecolor{navy}{rgb}{0,0,0.502}
\definecolor{pink}{rgb}{1,0.753,0.796}
\definecolor{seagreen}{rgb}{0.18,0.545,0.341}
\definecolor{turquoise}{rgb}{0.251,0.878,0.816}
\definecolor{violet}{rgb}{0.933,0.51,0.933}
\definecolor{darkblue}{rgb}{0,0,0.545}
\definecolor{darkcyan}{rgb}{0,0.545,0.545}
\definecolor{darkgray}{rgb}{0.663,0.663,0.663}
\definecolor{darkgreen}{rgb}{0,0.392,0}
\definecolor{darkmagenta}{rgb}{0.545,0,0.545}
\definecolor{darkorange}{rgb}{1,0.549,0}
\definecolor{darkred}{rgb}{0.545,0,0}
\definecolor{lightblue}{rgb}{0.678,0.847,0.902}
\definecolor{lightcyan}{rgb}{0.878,1,1}
\definecolor{lightgray}{rgb}{0.827,0.827,0.827}
\definecolor{lightgreen}{rgb}{0.565,0.933,0.565}
\definecolor{lightyellow}{rgb}{1,1,0.878}
\definecolor{black}{rgb}{0,0,0}
\definecolor{white}{rgb}{1,1,1}
\begin{tikzpicture}[ipe stylesheet]
  \draw[red, ipe pen fat]
    (64, 704)
     -- (96, 704)
     -- (96, 704);
  \pic[ipe mark large, fill=darkgray]
     at (64, 704) {ipe fdisk};
  \pic[ipe mark large, fill=darkgray]
     at (96, 704) {ipe fdisk};
  \filldraw[draw=red, ipe pen fat, fill=darkgray, ipe opacity 75]
    (144, 736)
     -- (128, 704)
     -- (128, 704);
  \filldraw[ipe pen fat, fill=white]
    (128, 704)
     -- (160, 704);
  \filldraw[red, ipe pen fat]
    (144, 736)
     -- (160, 704);
  \pic[ipe mark large, fill=darkgray]
     at (144, 736) {ipe fdisk};
  \pic[ipe mark large, fill=white]
     at (128, 704) {ipe fdisk};
  \pic[ipe mark large, fill=white]
     at (160, 704) {ipe fdisk};
  \filldraw[ipe pen fat, fill=white]
    (192, 704)
     -- (224, 704);
  \filldraw[ipe pen fat, fill=white]
    (192, 704)
     -- (208, 736);
  \filldraw[ipe pen fat, fill=white]
    (208, 736)
     -- (224, 704);
  \filldraw[draw=red, ipe pen fat, fill=darkgray]
    (208, 736)
     -- (240, 736);
  \filldraw[draw=red, ipe pen fat, fill=darkgray]
    (192, 704)
     -- (240, 736);
  \filldraw[draw=red, ipe pen fat, fill=black]
    (224, 704)
     -- (240, 736);
  \pic[ipe mark large, fill=white]
     at (192, 704) {ipe fdisk};
  \pic[ipe mark large, fill=white]
     at (208, 736) {ipe fdisk};
  \pic[ipe mark large, fill=white]
     at (224, 704) {ipe fdisk};
  \pic[ipe mark large, fill=darkgray]
     at (240, 736) {ipe fdisk};
  \node[ipe node, font=\Large]
     at (64, 672) {$G_1$};
  \node[ipe node, font=\Large]
     at (128, 672) {$G_2$};
  \node[ipe node, font=\Large]
     at (192, 672) {$G_3$};
  \node[ipe node, font=\Huge]
     at (256, 704) {. . .};
  \node[ipe node, font=\Large]
     at (320, 672) {$G_d$};
  \node[ipe node, font=\Large]
     at (323.895, 717.398) {$G_{d-1}$};
  \draw[ipe pen fat, ipe dash dashed]
    (362.6667, 730.6667)
     .. controls (362.6667, 741.3333) and (341.3333, 746.6667) .. (330.6667, 744)
     .. controls (320, 741.3333) and (320, 730.6667) .. (320, 725.3333)
     .. controls (320, 720) and (320, 720) .. (320, 714.6667)
     .. controls (320, 709.3333) and (320, 698.6667) .. (330.6667, 701.3333)
     .. controls (341.3333, 704) and (362.6667, 720) .. cycle;
  \filldraw[draw=red, ipe pen fat, fill=darkgray]
    (358.3048, 719.1079)
     -- (383.7055, 704.3703);
  \pic[ipe mark large, fill=darkgray]
     at (384, 704) {ipe fdisk};
  \filldraw[draw=red, ipe pen fat, fill=darkgray]
    (348.8779, 710.1407)
     -- (383.7886, 703.0978);
  \filldraw[draw=red, ipe pen fat, fill=darkgray]
    (364.0388, 730.0229)
     -- (385.3779, 704.2612);
  \filldraw[draw=red, ipe pen fat, fill=darkgray]
    (341.3156, 704.9987)
     -- (383.412, 702.607);
  \pic[ipe mark large, fill=darkgray]
     at (384, 704) {ipe fdisk};
\end{tikzpicture}
\caption{Recursive definition of $G_d$.}
\label{fig:graphicmatroid}
\end{figure}
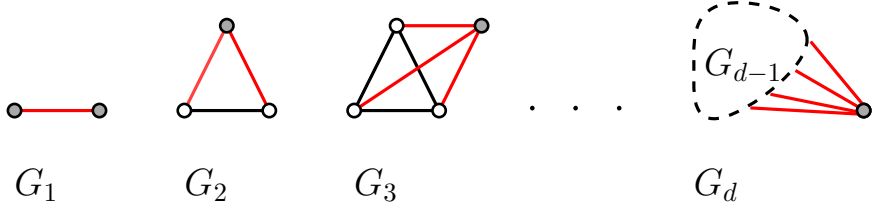

Consider now the arm-pulling schedule constructed by the greedy strategy. Let $T_p = E_p \setminus E_{p-1}$ be the new edges added at each step $p \in [d]$ in the recursive definition of $G_d$ (assuming that $E_0 = \emptyset$). Notice that for any integers $d \geq p_1 > p_2 \geq 1$ the edges of $T_{p_1}$ correspond to arms of higher mean reward than the edges of $T_{p_2}$. Therefore, the algorithm produces a periodic schedule of period $d$ as follows: Initially, the algorithm plays the $d$ arms of group $T_d$, collecting reward $d\left( 1 - \frac{\epsilon}{d}\right) = d - \epsilon$. Notice that, by construction, these edges form a spanning tree in $G_{d}$ and, thus, no additional arm can be played at the same time step. In the second time step of the period, the arms of $T_d$ are blocked and the algorithm plays the arms of $T_{d-1}$ collecting $d-1-\epsilon$ reward. Again, this is the maximum reward independent set of $G_d$ among the available arms. The algorithm proceeds similarly in the following steps and collects an average reward of 
$$
\frac{\sum^d_{p=1}(p-\epsilon)}{d} = \frac{d\cdot(d+1)/2-d\epsilon}{d} = \frac{d+1}{2} - \epsilon.
$$
In the above example, the optimal arm-pulling sequence is to play at each time $t \in [T]$, one arm of each group $T_p$ for $p \in [d]$. Notice that by construction of the delays and at each time step, there always exists at least one arm per group that is available. Moreover, by definition of the graph $G_d$, any such selection of arms never contains a circuit and, thus, it is an independent set of the graphic matroid. The expected reward collected by the optimal algorithm at each step is $d - \epsilon\sum_{p \in [d]}\frac{1}{p} = d - \epsilon H(d)$, were $H(d) = \sum_{p \in [d]}\frac{1}{p}$. 

In the above example, the ratio between the average reward collected by the greedy strategy and the optimal reward for $\epsilon \to 0$ becomes
$$
\lim_{\epsilon \to 0} \frac{\frac{d+1}{2} - \epsilon}{d - \epsilon H(d)} = \frac{1}{2} + \frac{1}{2d}.
$$
Therefore, by choosing large enough $d$, we can bring the approximation ratio of the above example arbitrarily close to $\frac{1}{2}$.
\end{proof}

\end{document}